\documentclass{article}

\usepackage[final,nonatbib]{neurips_2019}

\usepackage[utf8]{inputenc} 
\usepackage[T1]{fontenc}    
\usepackage{hyperref}       
\usepackage{url}            
\usepackage{booktabs}       
\usepackage{amsfonts}       
\usepackage{nicefrac}       
\usepackage{microtype}      
\usepackage{mathtools}
\usepackage{amsthm}
\usepackage{amssymb}
\usepackage{amsmath}
\usepackage{color}
\usepackage[skip=6pt]{caption}  
\usepackage{wrapfig}
\usepackage{multirow}
\usepackage{array}
\usepackage{subcaption}
\usepackage{float}
\usepackage{enumerate}
\usepackage{bbm}
\usepackage{graphicx}
\usepackage[dvipsnames]{xcolor}
\usepackage{wasysym}
\usepackage[round]{natbib}
\usepackage{arydshln}

\newcommand{\E}{\mathbb{E}}
\newcommand{\R}{\mathbb{R}}
\newcommand{\bigO}{\mathcal{O}}

\definecolor{mydarkblue}{rgb}{0,0.08,0.45}
\hypersetup{
    colorlinks=true,
    linkcolor=mydarkblue,
    citecolor=mydarkblue,
    filecolor=mydarkblue,
    urlcolor=mydarkblue}

\newtheorem{theorem}{Theorem}
\newtheorem{lemma}[theorem]{Lemma}
\theoremstyle{definition}

\newcommand{\ra}[1]{\renewcommand{\arraystretch}{#1}}
\newcommand{\Lip}{\mathrm{Lip}}
\DeclareMathOperator{\tr}{tr}

\title{Residual Flows for Invertible Generative Modeling}

\author{%
  Ricky T. Q. Chen$^{1,3}$, Jens Behrmann$^2$, David Duvenaud$^{1,3}$, J\"orn-Henrik Jacobsen$^{1,3}$\\
  University of Toronto$^1$, University of Bremen$^2$, Vector Institute$^3$
  \vspace{1.5mm}\\
  \texttt{rtqichen@cs.toronto.edu},\;
  \texttt{jensb@uni-bremen.de}\\
  \texttt{duvenaud@cs.toronto.edu},\;
  \texttt{j.jacobsen@vectorinstitute.ai}
}

\begin{document}

\maketitle

\begin{abstract}
Flow-based generative models parameterize probability distributions through an invertible transformation and can be trained by maximum likelihood.
Invertible residual networks provide a flexible family of transformations where only Lipschitz conditions rather than strict architectural constraints are needed for enforcing invertibility.
However, prior work trained invertible residual networks for density estimation by relying on biased log-density estimates whose bias increased with the network's expressiveness.
We give a tractable unbiased estimate of the log density using a ``Russian roulette'' estimator, and reduce the memory required during training by using an alternative infinite series for the gradient.
Furthermore, we improve invertible residual blocks by proposing the use of activation functions that avoid derivative saturation and generalizing the Lipschitz condition to induced mixed norms.
The resulting approach, called Residual Flows, achieves state-of-the-art performance on density estimation amongst flow-based models, and outperforms networks that use coupling blocks at joint generative and discriminative modeling.
\end{abstract}

\section{Introduction}

\begin{wrapfigure}[20]{r}{0.4\linewidth}
	\vspace{-1em}
	\centering
	\begin{subfigure}{0.5\linewidth}
		\captionsetup{justification=centering}
		\centering
		\includegraphics[width=0.6\linewidth]{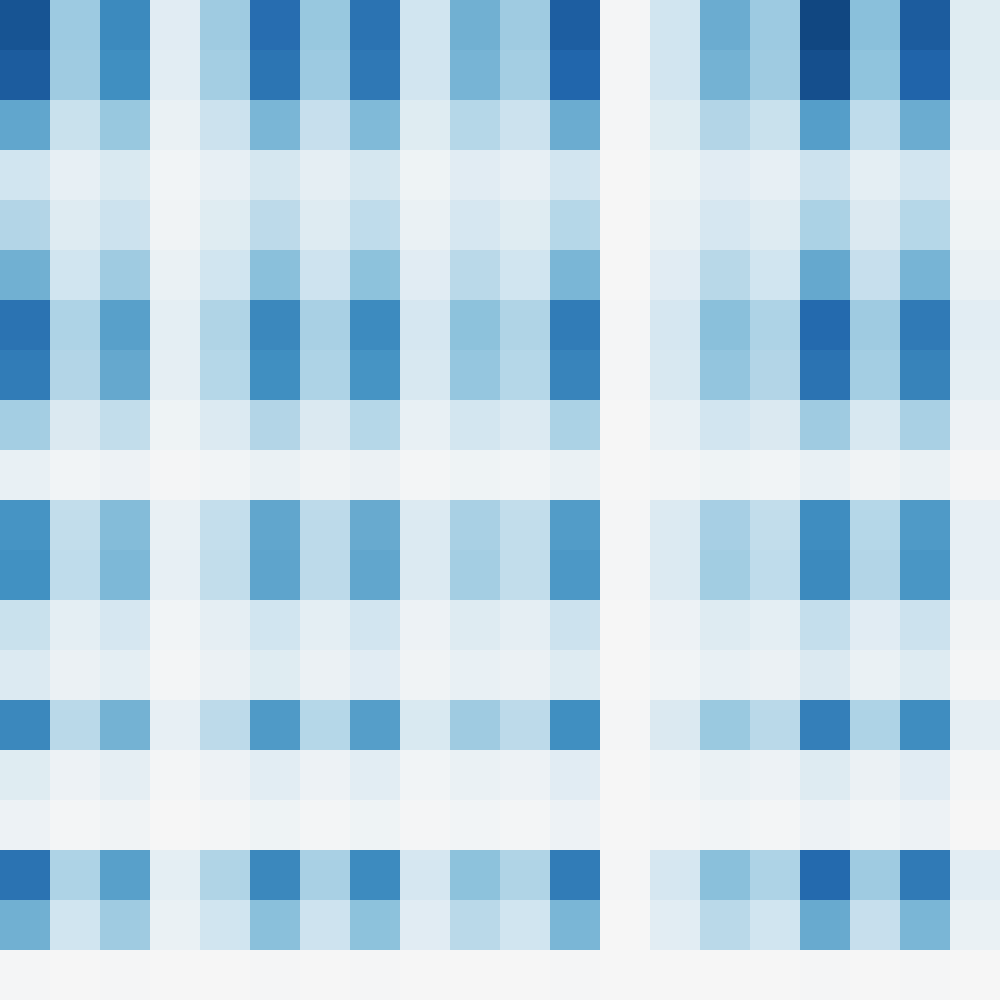}
		\caption{Det. Identities\\(Low Rank)}
	\end{subfigure}%
	\begin{subfigure}{0.5\linewidth}
		\captionsetup{justification=centering}
		\centering
		\includegraphics[width=0.6\linewidth]{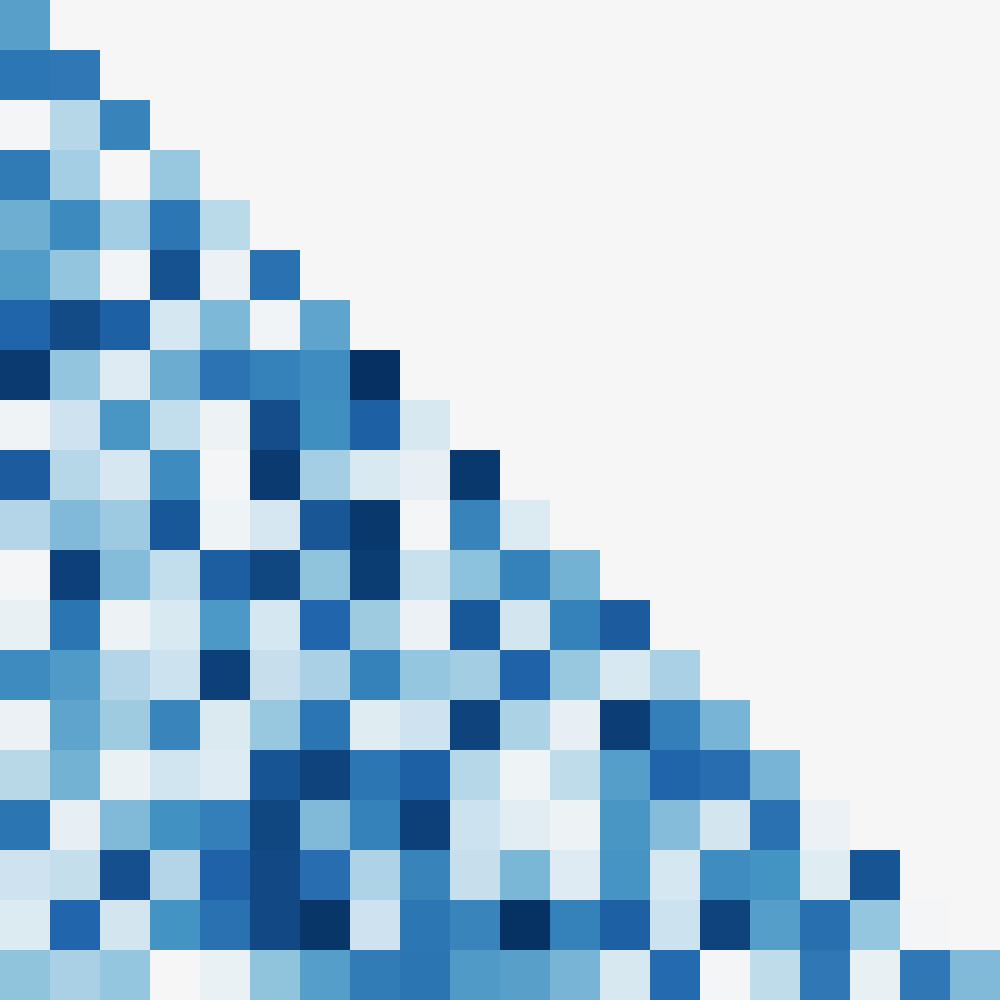}
		\caption{Autoregressive\\(Lower Triangular)}
	\end{subfigure}\\
	\vspace{1em}
	\begin{subfigure}{0.5\linewidth}
		\captionsetup{justification=centering}
		\centering
		\includegraphics[width=0.6\linewidth]{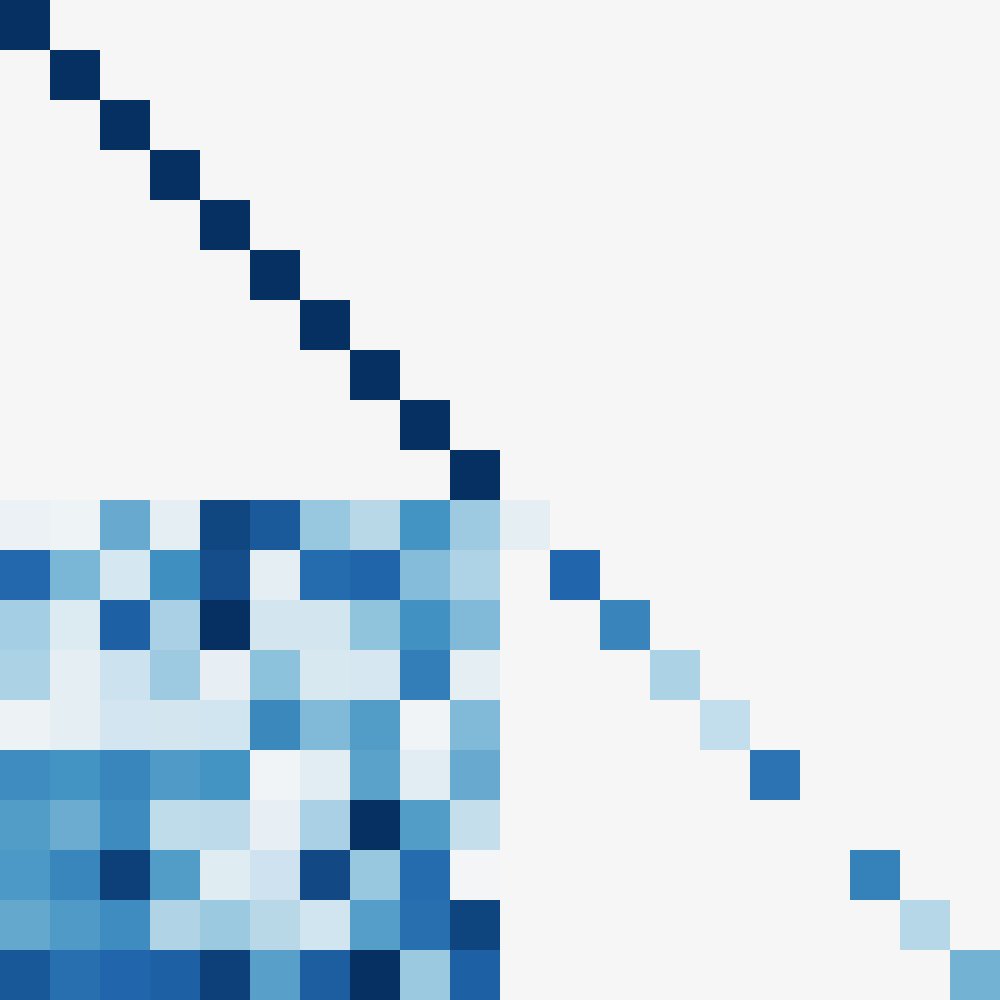}
		\caption{Coupling\\(Structured Sparsity)}
	\end{subfigure}%
	\begin{subfigure}{0.5\linewidth}
		\captionsetup{justification=centering}
		\centering
		\includegraphics[width=0.6\linewidth]{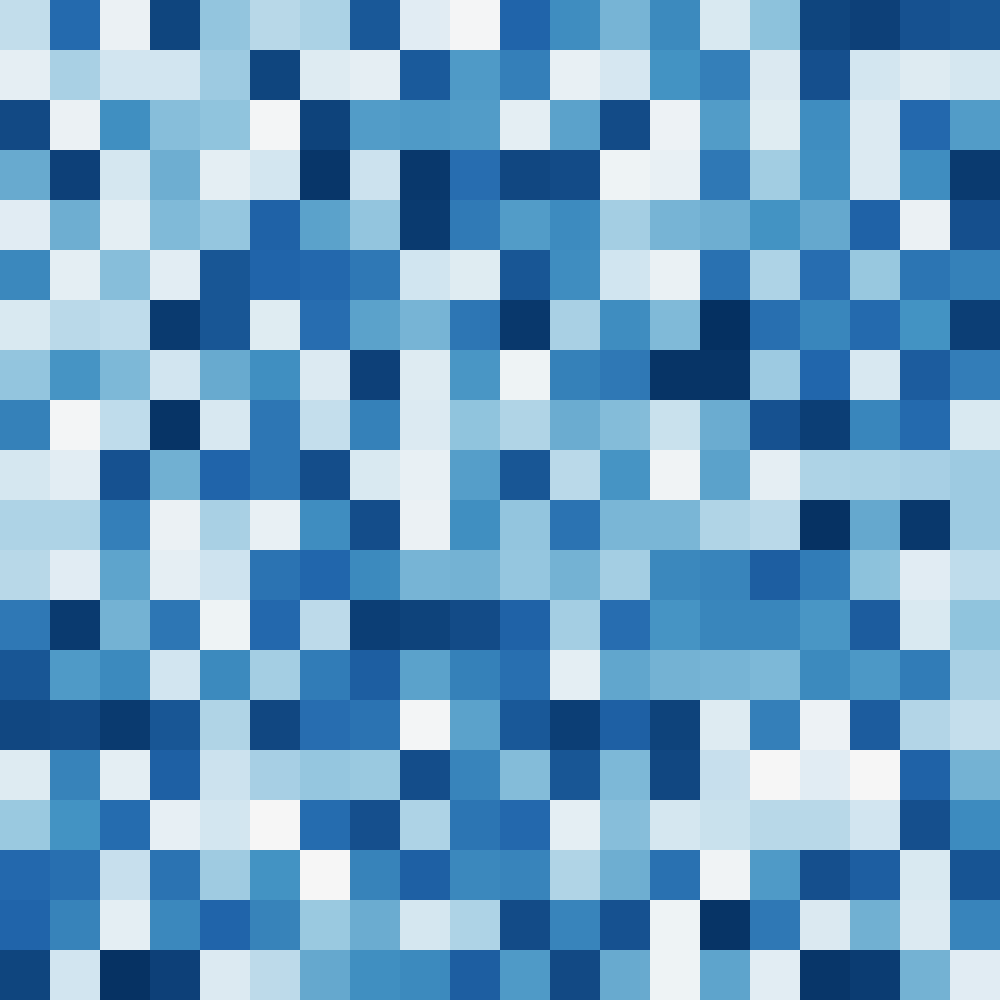}
		\caption{\textbf{Unbiased Est.}\\(Free-form)}
	\end{subfigure}%
	\caption{\textbf{Pathways to designing scalable normalizing flows} and their enforced Jacobian structure. Residual Flows fall under unbiased estimation with free-form Jacobian.}
	\label{fig:nf}
\end{wrapfigure}
Maximum likelihood is a core machine learning paradigm that poses learning as a distribution alignment problem. However, 
it is often unclear what family of distributions should be used to fit high-dimensional continuous data. In this regard, the change of variables theorem offers an appealing way to construct flexible distributions that allow tractable exact sampling and efficient evaluation of its density. This class of models is generally referred to as invertible or flow-based generative models~\citep{deco1995nonlinear,rezende2015variational}.

With invertibility as its core design principle, flow-based models (also referred to as normalizing flows) have shown to be capable of generating realistic images~\citep{kingma2018glow} and can achieve density estimation performance on-par with competing state-of-the-art approaches~\citep{ho2019flowpp}. In applications, they have been applied to study adversarial robustness~\citep{jacobsen2018excessive} and are used to train hybrid models with both generative and classification capabilities~\citep{nalisnick2019hybrid} using a weighted maximum likelihood objective.

Existing flow-based models~\citep{rezende2015variational,kingma2016improved,dinh2014nice,chen2018neural} make use of restricted transformations with sparse or structured Jacobians (Figure~\ref{fig:nf}). These allow efficient computation of the log probability under the model but at the cost of architectural engineering.
Transformations that scale to high-dimensional data rely on specialized architectures such as coupling blocks~\citep{dinh2014nice,dinh2016density} or solving an ordinary differential equation~\citep{grathwohl2019ffjord}. Such approaches have a strong inductive bias that can hinder their application in other tasks, such as learning representations that are suitable for both generative and discriminative tasks.

Recent work by \citet{behrmann2019} showed that residual networks~\citep{he2016deep} can be made invertible by simply enforcing a Lipschitz constraint, allowing to use a very successful discriminative deep network architecture for unsupervised flow-based modeling. Unfortunately, the density evaluation requires computing an infinite series. The choice of a fixed truncation estimator used by \citet{behrmann2019} leads to substantial bias that is tightly coupled with the expressiveness of the network, and cannot be said to be performing maximum likelihood as bias is introduced in the objective and gradients.

In this work, we introduce Residual Flows, a flow-based generative model that produces an unbiased estimate of the log density and has memory-efficient backpropagation through the log density computation. This allows us to use expressive architectures and train via maximum likelihood. Furthermore, we propose and experiment with the use of activations functions that avoid derivative saturation and induced mixed norms for Lipschitz-constrained neural networks.

\section{Background}

\paragraph{Maximum likelihood estimation.} To perform maximum likelihood with stochastic gradient descent, it is sufficient to have an unbiased estimator for the gradient as
\begin{equation}\label{eq:ml}
    \nabla_\theta D_\textnormal{KL}(p_{\textnormal{data}} \;||\; p_\theta) = \nabla_\theta \E_{x \sim p_{\textnormal{data}}(x)} \left[ \log p_\theta(x) \right] =
    \E_{x \sim p_{\textnormal{data}}(x)} \left[ \nabla_\theta \log p_\theta(x) \right],
\end{equation}
where $p_\textnormal{data}$ is the unknown data distribution which can be sampled from and $p_\theta$ is the model distribution. An unbiased estimator of the gradient also immediately follows from an unbiased estimator of the log density function, $\log p_\theta(x)$.

\paragraph{Change of variables theorem.} With an invertible transformation $f$, the change of variables
\begin{equation}\label{eq:cov}
    \log p(x) = \log p(f(x)) + \log \left| \det \frac{df(x)}{dx} \right|
\end{equation}
captures the change in density of the transformed samples. A simple base distribution such as a standard normal is often used for $\log p(f(x))$. Tractable evaluation of \eqref{eq:cov} allows flow-based models to be trained using the maximum likelihood objective~\eqref{eq:ml}. In contrast, variational autoencoders~\citep{kingma2013auto} can only optimize a stochastic lower bound, and generative adversial networks~\citep{goodfellow2014generative} require an extra discriminator network for training.

\paragraph{Invertible residual networks (i-ResNets).} Residual networks are composed of simple transformations $y = f(x) = x + g(x)$. \citet{behrmann2019} noted that this transformation is invertible by the Banach fixed point theorem if $g$ is contractive, i.e. with Lipschitz constant strictly less than unity, which was enforced using spectral normalization~\citep{miyato2018spectral,gouk2018regularisation}.

Applying i-ResNets to the change-of-variables \eqref{eq:cov}, the identity 
\begin{equation}\label{eq:cov_res}
    \log p(x) = \log p(f(x)) + \tr\left( \sum_{k=1}^\infty  \frac{(-1)^{k+1}}{k} [J_g(x)]^k \right)
\end{equation}
was shown, where $J_g(x) = \frac{dg(x)}{dx}$. Furthermore, the Skilling-Hutchinson estimator~\citep{skilling1989eigenvalues,hutchinson1990stochastic} was used to estimate the trace in the power series. \citet{behrmann2019} used a fixed truncation to approximate the infinite series in \eqref{eq:cov_res}. 
However, this na\"ive approach has a bias that grows with the number of dimensions of $x$ and the Lipschitz constant of $g$, as both affect the convergence rate of this power series.
As such, the fixed truncation estimator requires a careful balance between bias and expressiveness, and cannot scale to higher dimensional data.
Without decoupling the objective and estimation bias, i-ResNets end up optimizing for the bias without improving the actual maximum likelihood objective~(see Figure~\ref{fig:bias}).


\section{Residual Flows}

\subsection{Unbiased Log Density Estimation for Maximum Likelihood Estimation}

Evaluation of the exact log density function $\log p_\theta(\cdot)$ in \eqref{eq:cov_res} requires infinite time due to the power series. Instead, we rely on randomization to derive an unbiased estimator that can be computed in finite time (with probability one) based on an existing concept~\citep{kahn1955use}.

To illustrate the idea, let $\Delta_k$ denote the $k$-th term of an infinite series, and suppose we always evaluate the first term then flip a coin $b \sim \textnormal{Bernoulli}(q)$ to determine whether we stop or continue evaluating the remaining terms. By reweighting the remaining terms by $\frac{1}{1-q}$, we obtain an unbiased estimator
\begin{equation}
    \Delta_1 + \E \left[ \left(\frac{\sum_{k=2}^\infty \Delta_k}{1-q} \right) \mathbbm{1}_{b=0} + (0)\mathbbm{1}_{b=1}\right] = \Delta_1 + \frac{\sum_{k=2}^\infty \Delta_k}{1-q} (1-q) = \sum_{k=1}^\infty \Delta_k.
\end{equation}
Interestingly, whereas na\"ive computation would always use infinite compute, this unbiased estimator has probability $q$ of being evaluated in finite time. We can obtain an estimator that is evaluated in finite time with probability one by applying this process infinitely many times to the remaining terms. Directly sampling the number of evaluated terms, we obtain the appropriately named ``Russian roulette'' estimator~\citep{kahn1955use}
\begin{equation}\label{eq:rr_estimator}
    \sum_{k=1}^\infty \Delta_k = \E_{n\sim p(N)} \left[ \sum_{k=1}^n \frac{\Delta_k}{\mathbb{P}(N \geq k)} \right].
\end{equation}
We note that the explanation above is only meant to be an intuitive guide and not a formal derivation. The peculiarities of dealing with infinite quantities dictate that we must make assumptions on $\Delta_k$, $p(N)$, or both in order for the equality in \eqref{eq:rr_estimator} to hold. While many existing works have made different assumptions depending on specific applications of \eqref{eq:rr_estimator}, we state our result as a theorem where the only condition is that $p(N)$ must have support over all of the indices.
\begin{theorem}[Unbiased log density estimator]
\label{thm:unbiasedEst}
Let $f(x) = x + g(x)$ with $\Lip(g) < 1$ and $N$ be a random variable with support over the positive integers. Then
\begin{equation}
\label{eq:unbiased_logdet}
\begin{split}
    \log p(x) = \log p(f(x)) 
    + \E_{n, v} \left[ \sum_{k=1}^n \frac{(-1)^{k+1}}{k} \frac{v^T[J_g(x)^k]v}{\mathbb{P}(N \geq k)}\right],
\end{split}
\end{equation}
where $n\sim p(N)$ and $v \sim \mathcal{N}(0,I)$. 
\end{theorem}
Here we have used the Skilling-Hutchinson trace estimator~\citep{skilling1989eigenvalues,hutchinson1990stochastic} to estimate the trace of the matrices $J_g^k$. A detailed proof is given in Appendix \ref{sec:proofs}.

\begin{wrapfigure}[17]{r}{0.5\linewidth}
	\vspace{-0.3em}
	\centering
	\includegraphics[width=\linewidth, trim=18px 0 36px 25px, clip]{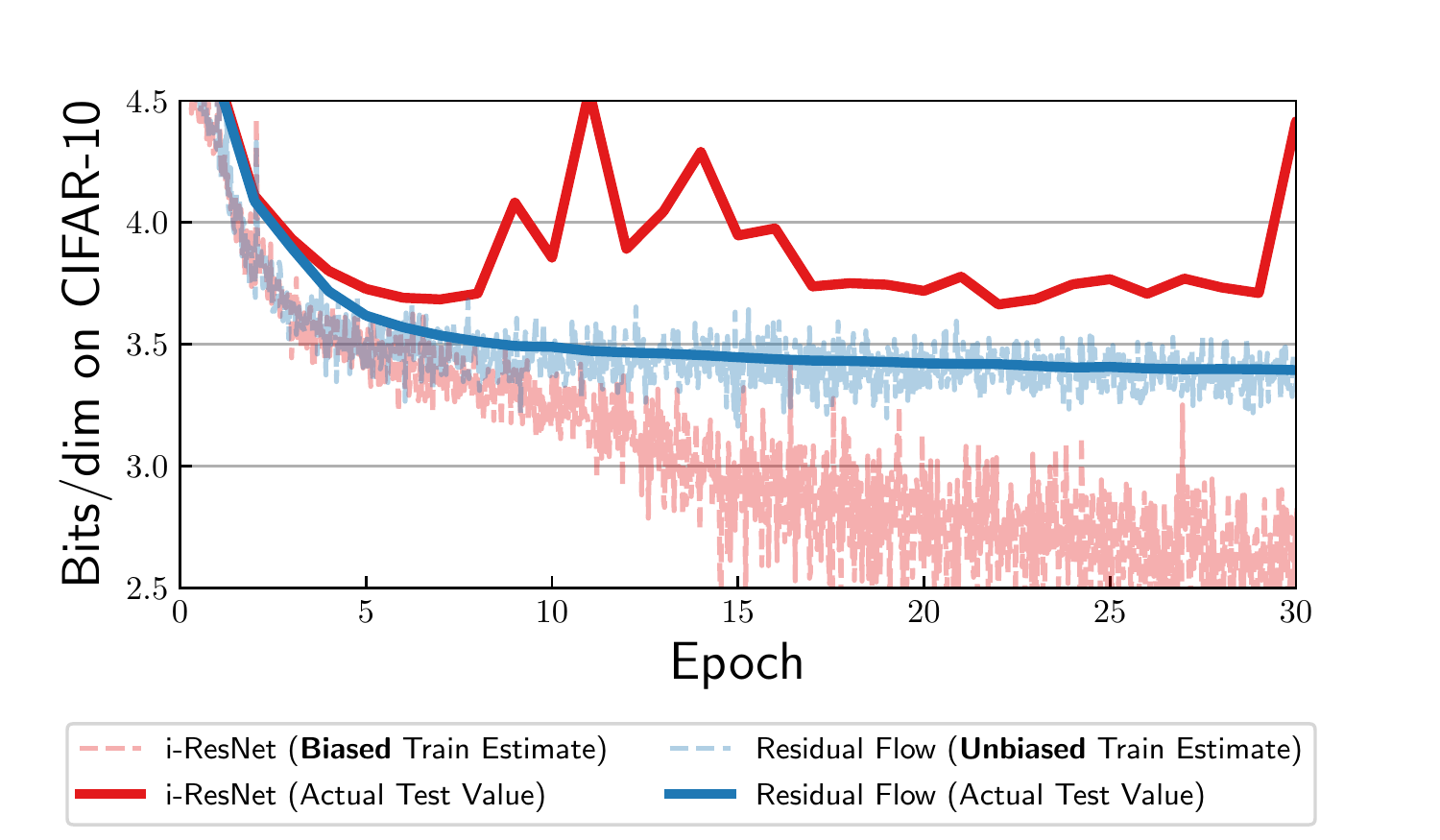}
	\caption{i-ResNets suffer from substantial bias when using expressive networks, whereas Residual Flows principledly perform maximum likelihood with unbiased stochastic gradients.}
	\label{fig:bias}
\end{wrapfigure}
Note that since $J_g$ is constrained to have a spectral radius less than unity, the power series converges exponentially. The variance of the Russian roulette estimator is small when the infinite series exhibits fast convergence~\citep{rhee2015unbiased,beatson2019efficient}, and in practice, we did not have to tune $p(N)$ for variance reduction. Instead, in our experiments, we compute two terms exactly and then use the unbiased estimator on the remaining terms with a single sample from $p(N) = \textnormal{Geom}(0.5)$. This results in an expected compute cost of $4$ terms, which is less than the $5$ to $10$ terms that \citet{behrmann2019} used for their biased estimator.

Theorem~\ref{thm:unbiasedEst} forms the core of Residual Flows, as we can now perform maximum likelihood training by backpropagating through \eqref{eq:unbiased_logdet} to obtain unbiased gradients. This allows us to train more expressive networks where a biased estimator would fail (Figure~\ref{fig:bias}). The price we pay for the unbiased estimator is variable compute and memory, as each sample of the log density uses a random number of terms in the power series.

\subsection{Memory-Efficient Backpropagation}\label{sec:gradient}

Memory can be a scarce resource, and running out of memory due to a large sample from the unbiased estimator can halt training unexpectedly. To this end, we propose two methods to reduce the memory consumption during training. 

To see how na\"ive backpropagation can be problematic, the gradient w.r.t.\ parameters $\theta$ by directly differentiating through the power series~\eqref{eq:unbiased_logdet} can be expressed as
\begin{equation}
\frac{\partial}{\partial \theta} \log\det \big(I + J_g(x,\theta)\big)
= \E_{n,v} \left[ \sum_{k=1}^n   \frac{(-1)^{k+1}}{k} \frac{\partial v^T (J_g(x,\theta)^k)v}{\partial \theta} \right].
\end{equation}
Unfortunately, this estimator requires each term to be stored in memory because $\nicefrac{\partial}{\partial \theta}$ needs to be applied to each term. The total memory cost is then $\bigO(n\cdot m)$ where $n$ is the number of computed terms and $m$ is the number of residual blocks in the entire network. This is extremely memory-hungry during training, and a large random sample of $n$ can occasionally result in running out of memory.

\paragraph{Neumann gradient series.} Instead, we can specifically express the gradients as a power series derived from a Neumann series (see Appendix \ref{app:gradient}). Applying the Russian roulette and trace estimators, we obtain the following theorem.

\begin{theorem}[Unbiased log-determinant gradient estimator]
\label{thm:unbiasedGrad}
Let $\Lip(g) < 1$ and $N$ be a random variable with support over positive integers. Then
\begin{align}\label{eq:neumannSeries}
\frac{\partial}{\partial \theta} \log\det \big(I + J_g(x,\theta)\big) 
= \E_{n, v}  \left[\left( \sum_{k=0}^n \frac{(-1)^k}{\mathbb{P}(N \geq k)} \; v^T J(x, \theta)^k \right) \frac{\partial (J_g(x,\theta))}{\partial \theta} v\right] ,
\end{align}
where $n\sim p(N)$ and $v \sim \mathcal{N}(0,I)$. 
\end{theorem}

As the power series in \eqref{eq:neumannSeries} does not need to be differentiated through, using this reduces the memory requirement by a factor of $n$. This is especially useful when using the unbiased estimator as the memory will be constant regardless of the number of terms we draw from $p(N)$.

\paragraph{Backward-in-forward: early computation of gradients.} We can further reduce memory by partially performing backpropagation during the forward evaluation. By taking advantage of $\log\det (I + J_g(x,\theta))$ being a scalar quantity, the partial derivative from the objective $\mathcal{L}$ is
\begin{align}
    \frac{\partial \mathcal{L}}{\partial \theta} = \underbrace{\frac{\partial \mathcal{L}}{\partial \log\det (I + J_g(x,\theta))}}_{\textnormal{scalar}} \underbrace{\frac{\partial \log\det (I + J_g(x,\theta))}{\partial \theta \vphantom{J_g}}}_{\textnormal{vector}}.
\end{align}

\begin{wrapfigure}[15]{r}{0.5\linewidth}
	\vspace{-0.8em}
	\centering
	\includegraphics[width=\linewidth,trim=8px 0 10px 10px, clip]{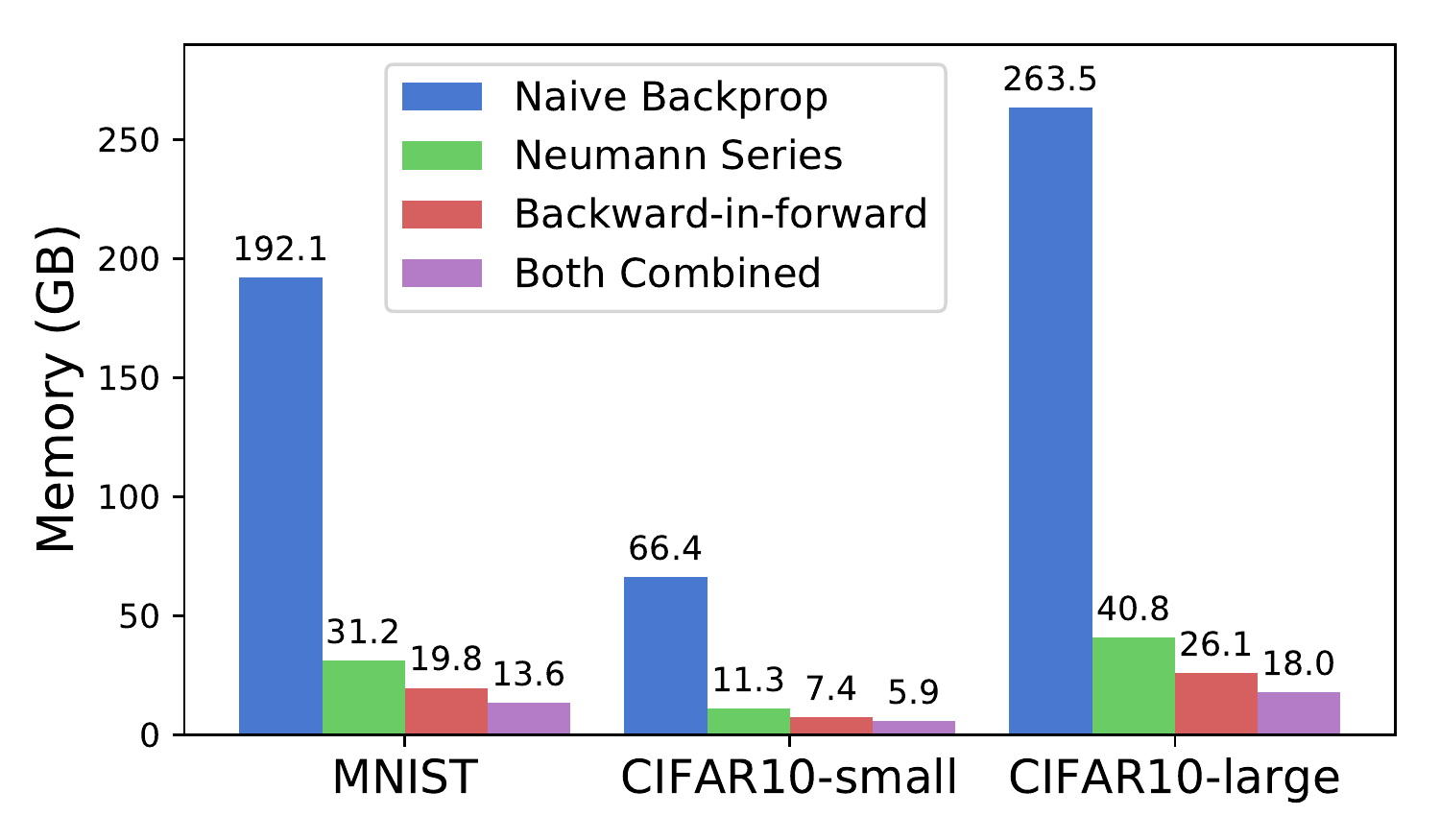}
	\caption{Memory usage (GB) per minibatch of 64 samples when computing $n$=10 terms in the corresponding power series. \textit{CIFAR10-small} uses immediate downsampling before any residual blocks.}
	\label{fig:memory}
\end{wrapfigure}
For every residual block, we compute $\nicefrac{\partial \log\det (I + J_g(x,\theta))}{\partial \theta}$ along with the forward pass, release the memory for the computation graph, then simply multiply by $\nicefrac{\partial \mathcal{L}}{\partial \log\det (I + J_g(x,\theta))}$ later during the main backprop. This reduces memory by another factor of $m$ to $\bigO(1)$ with negligible overhead. 

Note that while these two tricks remove the memory cost from backpropagating through the $\log \det$ terms, computing the path-wise derivatives from $\log p (f(x))$ still requires the same amount of memory as a single evaluation of the residual network. Figure~\ref{fig:memory} shows that the memory consumption can be enormous for na\"ive backpropagation, and using large networks would have been intractable. 

\subsection{Avoiding Derivative Saturation with the LipSwish Activation Function}

As the log density depends on the first derivatives through the Jacobian $J_g$, the gradients for training depend on second derivatives. Similar to the phenomenon of saturated activation functions, Lipschitz-constrained activation functions can have a derivative saturation problem. For instance, the ELU activation used by \citet{behrmann2019} achieves the highest Lipschitz constant when $\textnormal{ELU}'(z) = 1$, but this occurs when the second derivative is exactly zero in a very large region, implying there is a trade-off between a large Lipschitz constant and non-vanishing gradients.

We thus desire two properties from our activation functions $\phi(z)$:
\begin{enumerate}
    \item The first derivatives must be bounded as $|\phi'(z)| \leq 1$ for all $z$
    \item The second derivatives should not asymptotically vanish when $|\phi'(z)|$ is close to one.
\end{enumerate}
While many activation functions satisfy condition 1, most do not satisfy condition 2. We argue that the ELU and softplus activations are suboptimal due to derivative saturation. Figure \ref{fig:actfns} shows that when softplus and ELU saturate at regions of unit Lipschitz, the second derivative goes to zero, which can lead to vanishing gradients during training.

We find that good activation functions satisfying condition 2 are \emph{smooth and non-monotonic} functions, such as Swish~\citep{ramachandran2017searching}. However, Swish by default does not satisfy condition 1 as $\max_z |\frac{d}{dz} \textnormal{Swish}(z)| \apprle 1.1$. But scaling via
\begin{equation}
    \textnormal{LipSwish}(z) := \textnormal{Swish}(z)/1.1 = z \cdot \sigma(\beta z) / 1.1,
\end{equation}
where $\sigma$ is the sigmoid function, results in $\max_z |\frac{d}{dz} \textnormal{LipSwish}(z)| \leq 1$ for all values of $\beta$. LipSwish is a simple modification to Swish that exhibits a less than unity Lipschitz property. In our experiments, we parameterize $\beta$ to be strictly positive by passing it through softplus. Figure \ref{fig:actfns} shows that in the region of maximal Lipschitz, LipSwish does not saturate due to its non-monotonicity property.

\begin{figure}
\centering
\begin{subfigure}[b]{0.3\linewidth}
	\centering
	\includegraphics[width=\linewidth]{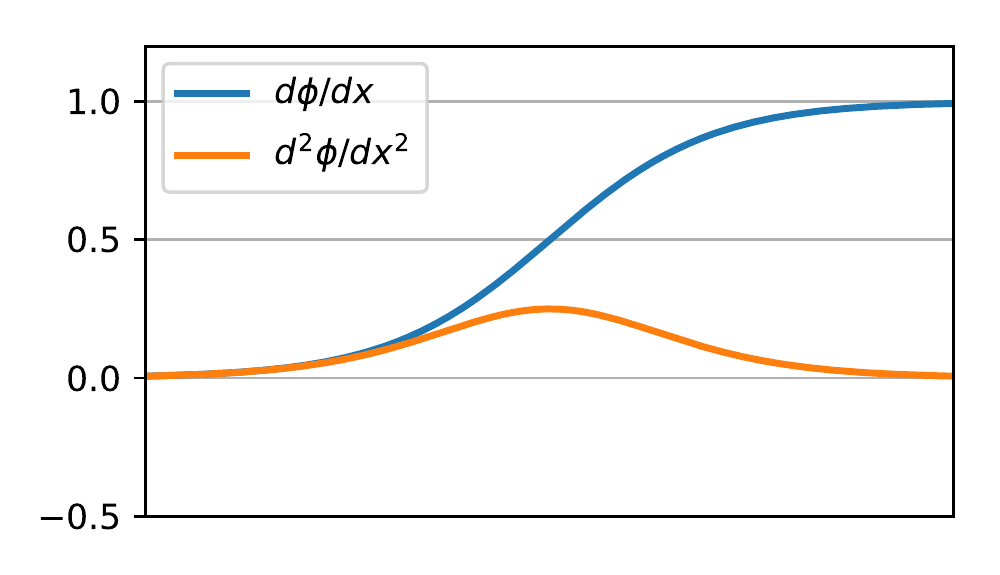}
	\vspace{-1.8em}
	\caption*{Softplus}
\end{subfigure}
\begin{subfigure}[b]{0.3\linewidth}
	\centering
	\includegraphics[width=\linewidth]{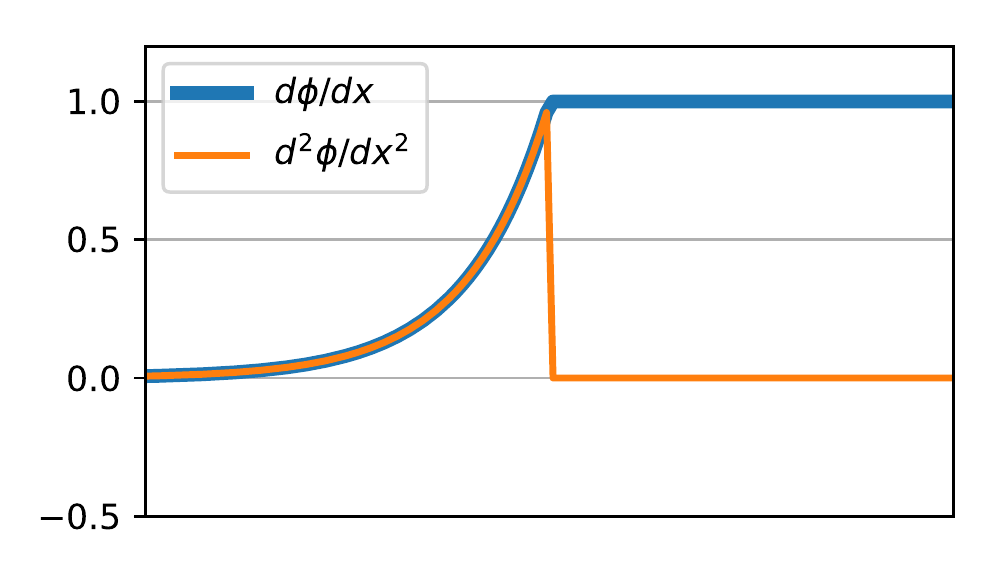}
	\vspace{-1.8em}
	\caption*{ELU}
\end{subfigure}
\begin{subfigure}[b]{0.3\linewidth}
	\centering
	\includegraphics[width=\linewidth]{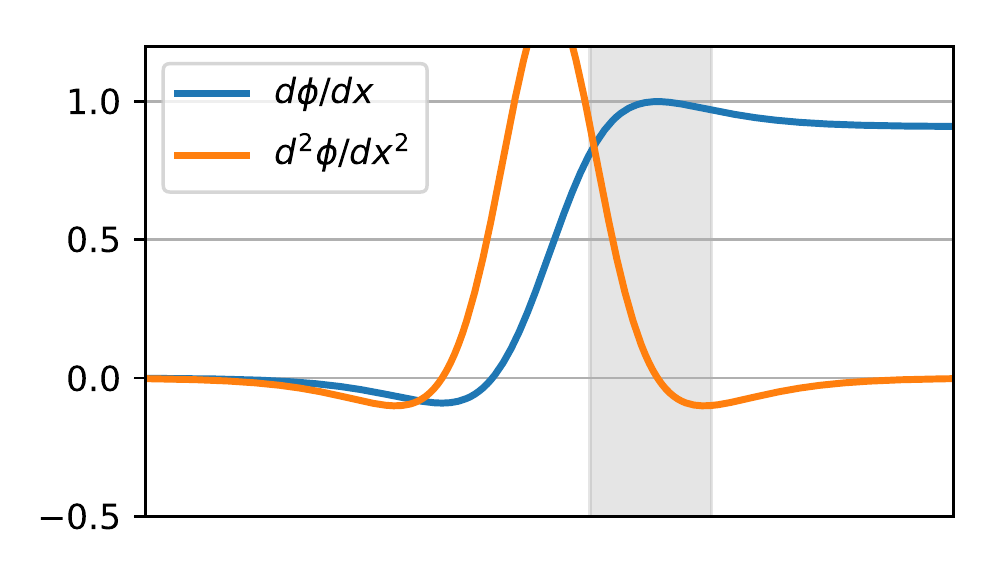}
	\vspace{-1.8em}
	\caption*{LipSwish}
\end{subfigure}
\caption{Common smooth Lipschitz activation functions $\phi$ usually have vanishing $\phi''$ when $\phi'$ is maximal. LipSwish has a non-vanishing $\phi''$ in the region where $\phi'$ is close to one.}
\label{fig:actfns}
\end{figure}

\section{Related Work}

\paragraph{Estimation of Infinite Series.} Our derivation of the unbiased estimator follows from the general approach of using a randomized truncation~\citep{kahn1955use}. This paradigm of estimation has been repeatedly rediscovered and applied in many fields, including solving of stochastic differential equations~\citep{mcleish2010general,rhee2012new,rhee2015unbiased}, ray tracing for rendering paths of light~\citep{arvo1990particle}, and estimating limiting behavior of optimization problems~\citep{tallec2017unbiasing,beatson2019efficient}, among many other applications.
Some recent works use Chebyshev polynomials to estimate the spectral functions of symmetric matrices~\citep{han2018chebyshev,adams2018estimating,aditya2018backprop}. These works estimate quantities that are similar to those presented in this work, but a key difference is that the Jacobian in our power series is not symmetric. Works that proposed the random truncation approach typically made assumptions on $p(N)$ in order for it to be applicable to general infinite series~\citep{mcleish2010general,rhee2015unbiased,han2018chebyshev}. Fortunately, since the power series in Theorems \ref{thm:unbiasedEst} and \ref{thm:unbiasedGrad} converge fast enough, we were able to make use of a different set of assumptions requiring only that $p(N)$ has sufficient support~(details in Appendix~\ref{sec:proofs}).

\paragraph{Memory-efficient Backpropagation.} The issue of computing gradients in a memory-efficient manner was explored by \citet{gomez2017reversible} and \citet{chang2018reversible} for residual networks with a coupling-based architecture devised by~\citet{dinh2014nice}, and explored by \citet{chen2018neural} for a continuous analogue of residual networks. These works focus on the path-wise gradients from the output of the network, whereas we focus on the gradients from the log-determinant term in the change of variables equation specifically for generative modeling. On the other hand, our approach shares some similarities with Recurrent Backpropagation \citep{Almeida87, Pineda87, liao2018reviving}, since both approaches leverage convergent dynamics to modify the derivatives.

\paragraph{Invertible Deep Networks.} Flow-based generative models are a density estimation approach which has invertibility as its core design principle \citep{rezende2015variational,deco1995nonlinear}. Most recent work on flows focuses on designing maximally expressive architectures while maintaining invertibility and tractable log determinant computation \citep{dinh2014nice,dinh2016density,kingma2018glow}. 
An alternative route has been taken by Continuous Normalizing Flows~\citep{chen2018neural} which make use of Jacobian traces instead of Jacobian determinants, provided that the transformation is parameterized by an ordinary differential equation. Invertible architectures are also of interest for discriminative problems, as their information-preservation properties make them suitable candidates for analyzing and regularizing learned representations \citep{jacobsen2018excessive}.

\section{Experiments}

\begin{table}
	\centering
	\ra{1.3}
	\setlength{\tabcolsep}{5pt}
	\caption{Results [bits/dim] on standard benchmark datasets for density estimation. In brackets are models that used ``variational dequantization''~\citep{ho2019flowpp}, which we don't compare against.}
	\label{tab:density}
	\resizebox{\textwidth}{!}{%
		\begin{tabular}{@{}llllll@{}}\toprule
			Model & \;MNIST\; & \;CIFAR-10\; & ImageNet 32 & ImageNet 64 & CelebA-HQ 256 \\
			\midrule
			Real NVP~{\scriptsize \citep{dinh2016density}} & 1.06 & 3.49 & 4.28 & 3.98& \;---\;\; \\
			Glow~{\scriptsize \citep{kingma2018glow}} & 1.05 & 3.35 & 4.09 & 3.81 & 1.03 \\
			FFJORD~{\scriptsize \citep{grathwohl2019ffjord}} & 0.99 & 3.40 & \;---\;\; & \;---\;\; &  \;---\;\;\\
			Flow++~{\scriptsize \citep{ho2019flowpp}} & \;---\;\; & 3.29 (3.09) & \;---\;\; (3.86) & \;---\;\; (3.69) &  \;---\;\;\\
			i-ResNet~{\scriptsize \citep{behrmann2019}} & 1.05 & 3.45 & \;---\;\; & \;---\;\; &  \;---\;\;\\
			\midrule
			Residual Flow (Ours) & \textbf{0.970} & \textbf{3.280} & \textbf{4.010} & \textbf{3.757} & \textbf{0.992} \\
			\bottomrule
	\end{tabular}}
\end{table}

\begin{figure}
    \centering
    \includegraphics[width=0.495\linewidth,trim=0 0 66 0, clip]{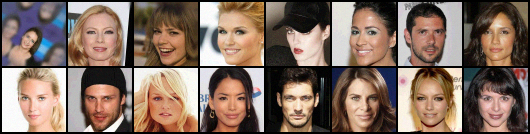}\hfill
    \includegraphics[width=0.495\linewidth,trim=0 0 66 0, clip]{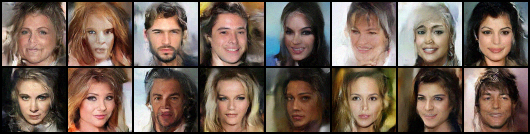}
    \caption{\textbf{Qualitative samples.} Real (left) and random samples (right) from a model trained on 5bit 64$\times$64 CelebA. The most visually appealing samples were picked out of 5 random batches.}
    \label{fig:celeba_samples}
\end{figure}

\subsection{Density \& Generative Modeling}
We use a similar architecture as \citet{behrmann2019}, except without the immediate invertible downsampling~\citep{dinh2016density} at the image pixel-level. Removing this substantially increases the amount of memory required~(shown in Figure~\ref{fig:memory}) as there are more spatial dimensions at every layer, but increases the overall performance. We also increase the bound on the Lipschitz constants of each weight matrix to $0.98$, whereas \citet{behrmann2019} used $0.90$ to reduce the error of the biased estimator. More detailed description of  architectures is in Appendix \ref{app:experiment_setup}.

Unlike prior works that use multiple GPUs, large batch sizes, and a few hundred epochs, Residual Flow models are trained with the standard batch size of 64 and converges in roughly 300-350 epochs for MNIST and CIFAR-10. Most network settings can fit on a single GPU~(see Figure~\ref{fig:memory}), though we use 4 GPUs in our experiments to speed up training. 
On CelebA-HQ, Glow had to use a batchsize of 1 per GPU with a budget of 40 GPUs whereas we trained our model using a batchsize of 3 per GPU and a budget of 4 GPUs, owing to the smaller model and memory-efficient backpropagation.

Table \ref{tab:density} reports the bits per dimension ($\log_2 p(x)/d$ where $x \in \R^d$) on standard benchmark datasets MNIST, CIFAR-10, downsampled ImageNet, and CelebA-HQ. We achieve competitive performance to state-of-the-art flow-based models on all datasets. 
For evaluation, we computed 20 terms of the power series~\eqref{eq:cov_res} and use the unbiased estimator~\eqref{eq:unbiased_logdet} to estimate the remaining terms. This reduces the standard deviation of the unbiased estimate of the test bits per dimension to a negligible level. 

Furthermore, it is possible to generalize the Lipschitz condition of Residual Flows to arbitrary p-norms and even mixed matrix norms. By learning the norm orders jointly with the model, we achieved a small gain of 0.003 bits/dim on CIFAR-10 compared to spectral normalization. In addition, we show that others norms like $p=\infty$ yielded constraints more suited for lower dimensional data. See Appendix \ref{app:generalized_sn} for a discussion on how to generalize the Lipschitz condition and an exploration of different norm-constraints for 2D problems and image data.

\subsection{Sample Quality}

We are also competitive with state-of-the-art flow-based models in regards to sample quality. Figure~\ref{fig:celeba_samples} shows random samples from the model trained on CelebA. Furthermore, samples from Residual Flow trained on CIFAR-10 are more globally coherent (Figure~\ref{fig:cifar10_samples}) than PixelCNN and variational dequantized Flow++, even though our likelihood is worse. 

\begin{wraptable}[15]{r}{0.33\linewidth}
	\centering
	\ra{1.2}
	\vspace{-1.11em}
	\setlength{\tabcolsep}{2pt}
	\caption{Lower FID implies better sample quality. $^*$Results taken from \citet{ostrovski2018autoregressive}.}
	\label{tab:fid_cifar10}
	\begin{tabular}{@{}lr@{}}\toprule
		Model & CIFAR10 FID\\
		\midrule
		PixelCNN$^*$ & 65.93 \\
		PixelIQN$^*$ & 49.46 \\
 		\hdashline
		\addlinespace[2pt]
		i-ResNet & 65.01 \\
		Glow & 46.90 \\
		Residual Flow \hspace{0.5em} & \textbf{46.37} \\
		\hdashline
		\addlinespace[2pt]
		DCGAN$^*$ & 37.11 \\
		WGAN-GP$^*$ & 36.40 \\
		\bottomrule
	\end{tabular}
\end{wraptable}
For quantitative comparison, we report FID scores~\citep{heusel2017gans} in Table~\ref{tab:fid_cifar10}. We see that Residual Flows significantly improves on i-ResNets and PixelCNN, and achieves slightly better sample quality than an official Glow model that has double the number of layers. It is well-known that visual fidelity and log-likelihood are not necessarily indicative of each other~\citep{theis2015note}, but we believe residual blocks may have a better inductive bias than coupling blocks or autoregressive architectures as generative models. More samples are in Appendix~\ref{app:samples}.

To generate visually appealing images, \citet{kingma2018glow} used temperature annealing (ie. sampling from $[p(x)]^{T^2}$ with $T<1$) to sample closer to the mode of the distribution, which helped remove artifacts from the samples and resulted in smoother looking images. However, this is done by reducing the entropy of $p(z)$ during sampling, which is only equivalent to temperature annealing if the change in log-density does not depend on the sample itself. Intuitively, this assumption implies that the mode of $p(x)$ and $p(z)$ are the same. 
As this assumption breaks for general flow-based models, including Residual Flows, we cannot use the same trick to sample efficiently from a temperature annealed model. Figure~\ref{fig:reduced_ent} shows the results of reduced entropy sampling on CelebA-HQ 256, but the samples do not converge to the mode of the distribution.
\begin{figure}
    \centering
    \begin{subfigure}[b]{0.49\textwidth}
        \centering
        \includegraphics[width=\linewidth, trim= 0 37px 0 241px, clip]{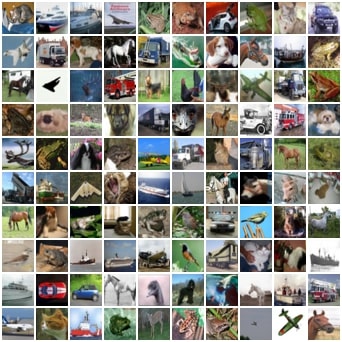}
        \vspace{-5mm}
        \caption*{CIFAR-10 Real Data}
    \end{subfigure}%
    \hfill
    \begin{subfigure}[b]{0.49\textwidth}
        \centering
        \includegraphics[width=\linewidth, trim= 0 71px 0 207px, clip]{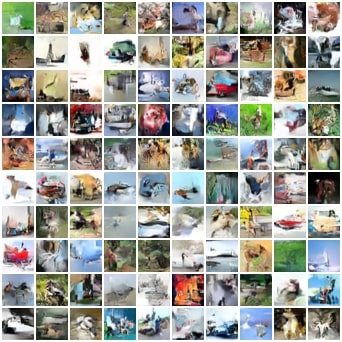}
        \vspace{-5mm}
        \caption*{Residual Flow (3.29 bits/dim)}
    \end{subfigure}\\
    \begin{subfigure}[b]{0.49\textwidth}
        \centering
        \includegraphics[width=\linewidth, trim = 0 314px 0 125px, clip]{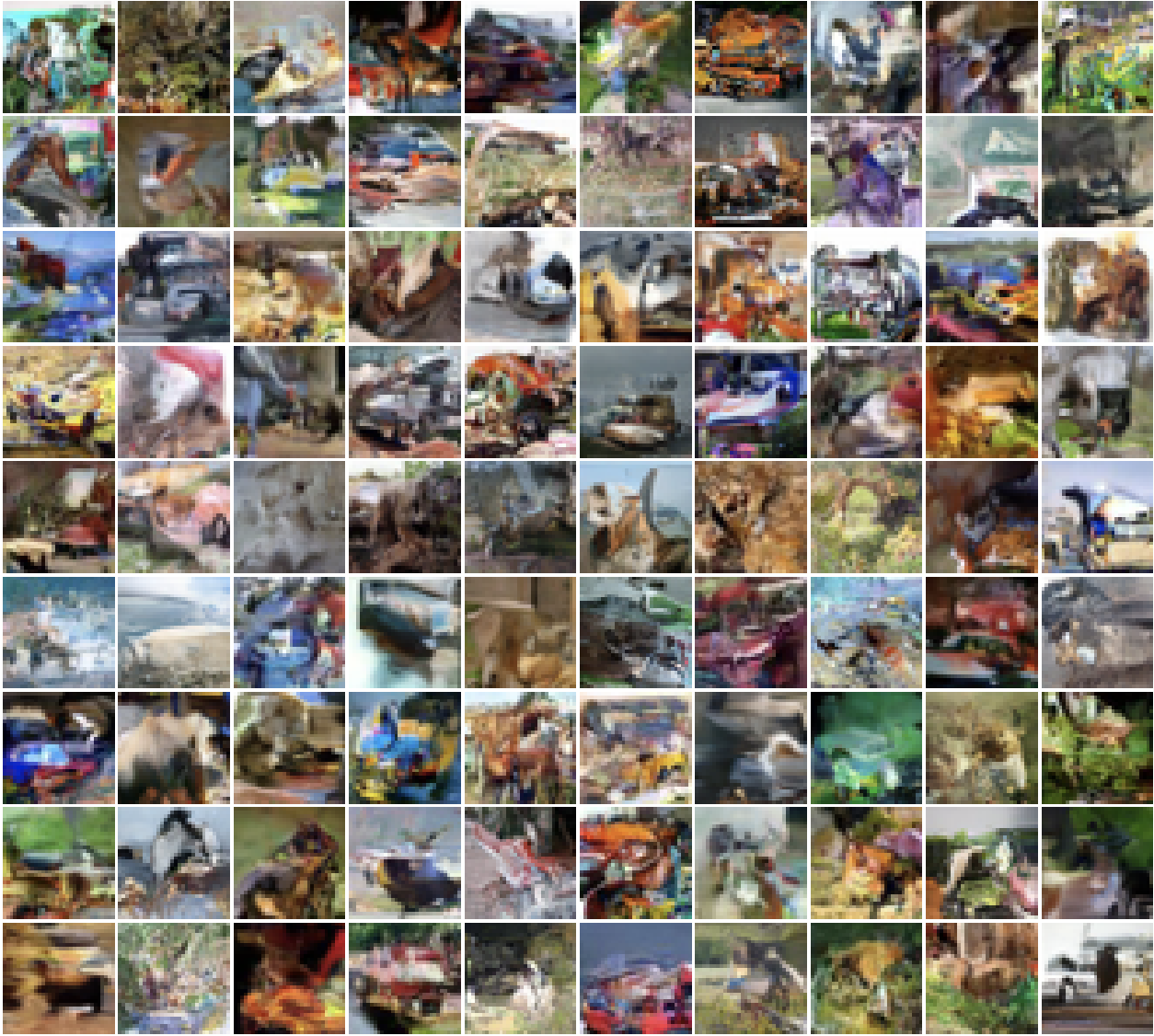}
        \vspace{-5mm}
        \caption*{PixelCNN (3.14 bits/dim)}
    \end{subfigure}%
    \hfill
    \begin{subfigure}[b]{0.49\textwidth}
        \centering
        \includegraphics[width=\linewidth, trim= 0 170px 0 103px, clip]{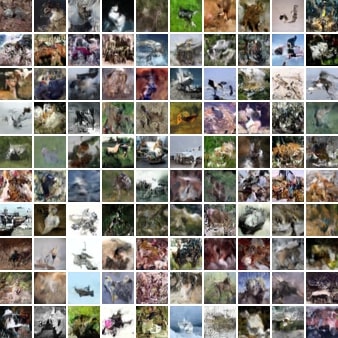}
        \vspace{-5mm}
        \caption*{Variational Dequantized Flow++ (3.08 bits/dim)}
    \end{subfigure}
    \caption{Random samples from Residual Flow are more globally coherent. PixelCNN~\citep{oord2016pixel} and Flow++ samples reprinted from \citet{ho2019flowpp}.}
    \label{fig:cifar10_samples}
\end{figure}

\begin{figure}
\centering
\begin{subfigure}[b]{0.12\linewidth}
	\centering
	\includegraphics[width=\linewidth]{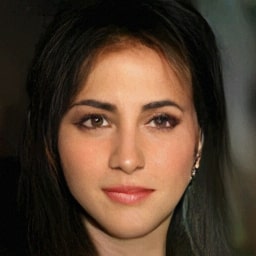}
\end{subfigure}%
\begin{subfigure}[b]{0.12\linewidth}
	\centering
	\includegraphics[width=\linewidth]{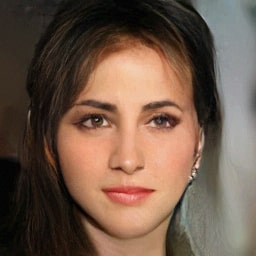}
\end{subfigure}%
\begin{subfigure}[b]{0.12\linewidth}
	\centering
	\includegraphics[width=\linewidth]{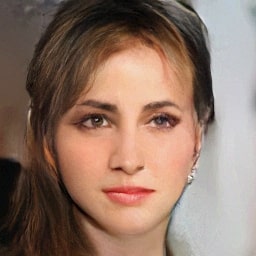}
\end{subfigure}%
\begin{subfigure}[b]{0.12\linewidth}
	\centering
	\includegraphics[width=\linewidth]{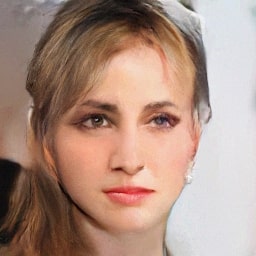}
\end{subfigure}\hspace{0.8em}
\begin{subfigure}[b]{0.12\linewidth}
	\centering
	\includegraphics[width=\linewidth]{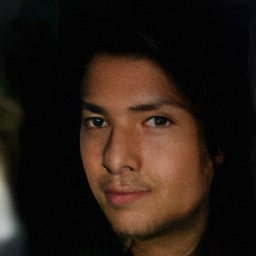}
\end{subfigure}%
\begin{subfigure}[b]{0.12\linewidth}
	\centering
	\includegraphics[width=\linewidth]{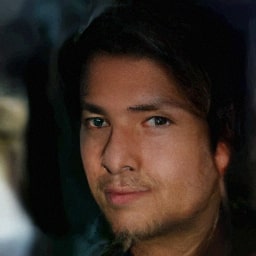}
\end{subfigure}%
\begin{subfigure}[b]{0.12\linewidth}
	\centering
	\includegraphics[width=\linewidth]{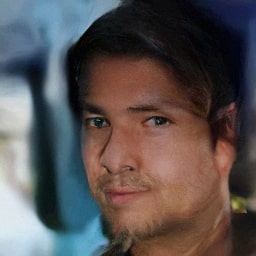}
\end{subfigure}%
\begin{subfigure}[b]{0.12\linewidth}
	\centering
	\includegraphics[width=\linewidth]{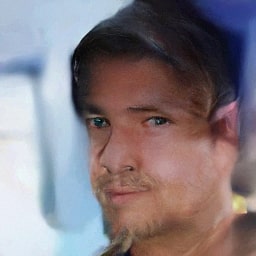}
\end{subfigure}

\begin{subfigure}[b]{0.12\linewidth}
	\centering
	\includegraphics[width=\linewidth]{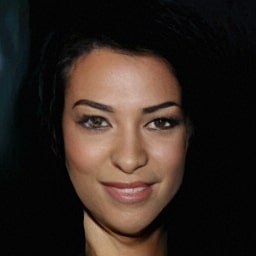}
	\vspace{-1.5em}
	\caption*{$T$=0.7}
\end{subfigure}%
\begin{subfigure}[b]{0.12\linewidth}
	\centering
	\includegraphics[width=\linewidth]{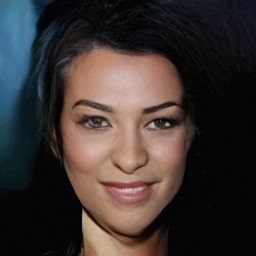}
	\vspace{-1.5em}
	\caption*{$T$=0.8}
\end{subfigure}%
\begin{subfigure}[b]{0.12\linewidth}
	\centering
	\includegraphics[width=\linewidth]{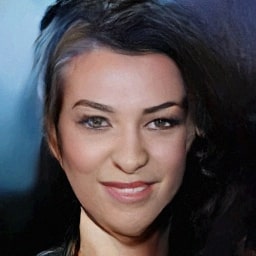}
	\vspace{-1.5em}
	\caption*{$T$=0.9}
\end{subfigure}%
\begin{subfigure}[b]{0.12\linewidth}
	\centering
	\includegraphics[width=\linewidth]{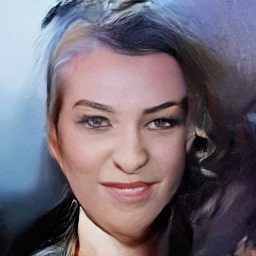}
	\vspace{-1.5em}
	\caption*{$T$=1.0}
\end{subfigure}\hspace{0.8em}
\begin{subfigure}[b]{0.12\linewidth}
	\centering
	\includegraphics[width=\linewidth]{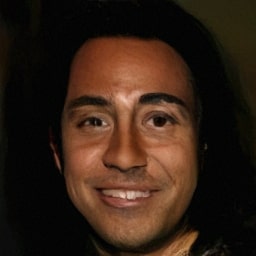}
	\vspace{-1.5em}
	\caption*{$T$=0.7}
\end{subfigure}%
\begin{subfigure}[b]{0.12\linewidth}
	\centering
	\includegraphics[width=\linewidth]{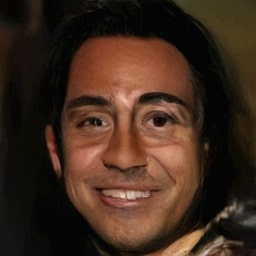}
	\vspace{-1.5em}
	\caption*{$T$=0.8}
\end{subfigure}%
\begin{subfigure}[b]{0.12\linewidth}
	\centering
	\includegraphics[width=\linewidth]{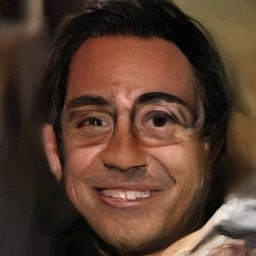}
	\vspace{-1.5em}
	\caption*{$T$=0.9}
\end{subfigure}%
\begin{subfigure}[b]{0.12\linewidth}
	\centering
	\includegraphics[width=\linewidth]{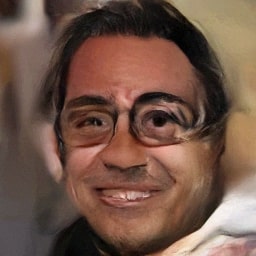}
	\vspace{-1.5em}
	\caption*{$T$=1.0}
\end{subfigure}
\caption{Reduced entropy sampling does not equate with proper temperature annealing for general flow-based models. Na\"ively reducing entropy results in samples that exhibit black hair and background, indicating that samples are not converging to the mode of the distribution.}
\label{fig:reduced_ent}
\end{figure}

\subsection{Ablation Experiments}
We report ablation experiments for the unbiased estimator and the LipSwish activation function in Table~\ref{tab:ablations}. Even in settings where the Lipschitz constant and bias are relatively low, we observe a significant improvement from using the unbiased estimator. Training the larger i-ResNet model on CIFAR-10 results in the biased estimator completely ignoring the actual likelihood objective altogether. In this setting, the biased estimate was lower than 0.8 bits/dim by 50 epochs, but the actual bits/dim wildly oscillates above 3.66 bits/dim and seems to never converge. Using LipSwish not only converges much faster but also results in better performance compared to softplus or ELU, especially in the high Lipschitz settings (Figure~\ref{fig:actfn} and Table~\ref{tab:ablations}). 

\begin{figure}
\centering
\begin{minipage}[b]{0.35\linewidth}
\centering
\centering
\includegraphics[width=\linewidth, trim=0 10px 0 5px, clip]{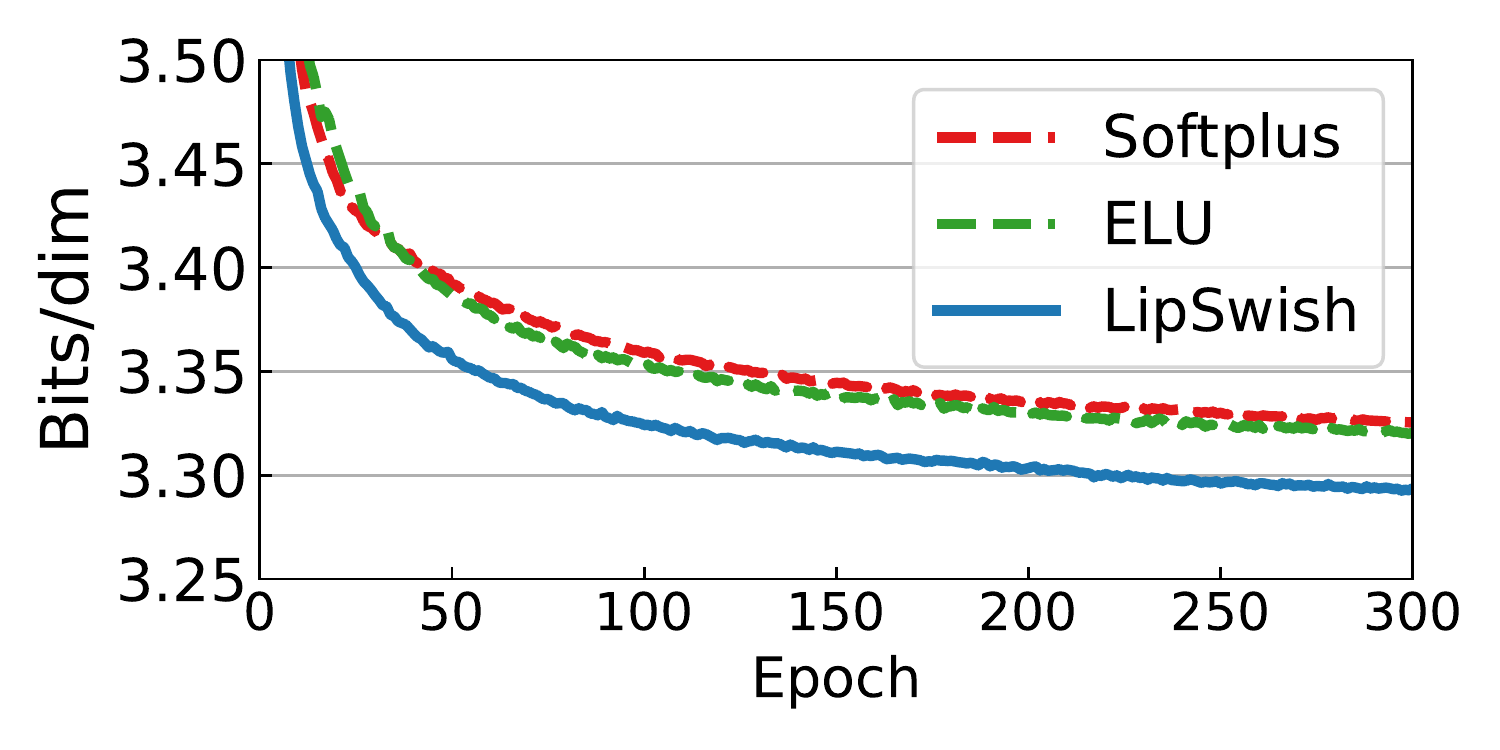}
\caption{Effect of activation\\ functions on CIFAR-10.}
\label{fig:actfn}
\end{minipage}\hfill
\begin{minipage}[b]{0.63\linewidth}
\centering
\captionsetup{type=table} 
\centering
\ra{1.2}
\resizebox{\textwidth}{!}{%
\begin{tabular}{@{}lccc@{}}\toprule
Training Setting & MNIST & CIFAR-10$^\dagger$ & CIFAR-10 \\
\midrule
i-ResNet + ELU & 1.05 & 3.45 & 3.66$\sim$4.78 \\
Residual Flow + ELU & 1.00 & 3.40 & 3.32 \\
Residual Flow + LipSwish & \textbf{0.97} & \textbf{3.39} & \textbf{3.28} \\
\bottomrule
\end{tabular}}
\vspace{0.5mm}
\caption{Ablation results. $^\dagger$Uses immediate downsampling before any residual blocks.}
\label{tab:ablations}
\end{minipage}
\end{figure}

\subsection{Hybrid Modeling}

\begin{table}[b]\centering
\ra{1.3}
\setlength{\tabcolsep}{2pt}
\caption{Comparison of residual vs. coupling blocks for the hybrid modeling task.}
\label{tab:hybrid}
\resizebox{\textwidth}{!}{%
\small
\begin{tabular}{@{} l ccccc c ccccc @{}}\toprule
& \multicolumn{5}{c}{\normalsize MNIST} & \hspace{1em} & \multicolumn{5}{c}{\normalsize SVHN} \\ 
\cmidrule(r){2-6} \cmidrule{8-12}
& $\lambda=0$ & \multicolumn{2}{c}{$\lambda=\nicefrac{1}{D}$} & \multicolumn{2}{c}{$\lambda=1$} & &
$\lambda=0$ & \multicolumn{2}{c}{$\lambda=\nicefrac{1}{D}$} & \multicolumn{2}{c}{$\lambda=1$} \\
\cmidrule{2-2} \cmidrule(lr){3-4} \cmidrule(r){5-6}
\cmidrule{8-8} \cmidrule(lr){9-10} \cmidrule{11-12}
{\normalsize Block Type} & Acc$\uparrow$ & \;\;\;BPD$\downarrow$ & Acc$\uparrow$ & \;\;\;BPD$\downarrow$ & Acc$\uparrow$ 
& & Acc$\uparrow$ & \;\;\;BPD$\downarrow$ & Acc$\uparrow$ & \;\;\;BPD$\downarrow$ & Acc$\uparrow$ \\
\midrule
\citet{nalisnick2019hybrid}\; & 99.33\% & \;\;\;1.26 & 97.78\% & \;\;\;$-$ & $-$ & & 95.74\% & \;\;\;2.40 & 94.77\% & $-$ & $-$ \\
\midrule

{\normalsize Coupling} 
& 99.50\% & \;\;\;1.18 & 98.45\% & \;\;\;1.04 & 95.42\% & 
& 96.27\% & \;\;\;2.73 & 95.15\% & \;\;\;2.21 & 46.22\% \\

{\normalsize \; + $1\times 1$ Conv} 
& \textbf{99.56\%} & \;\;\;1.15 & 98.93\% & \;\;\;1.03 & 94.22\% & 
& \textbf{96.72\%} & \;\;\;2.61 & 95.49\% & \;\;\;2.17 & 46.58\% \\

{\normalsize Residual} 
& 99.53\% & \;\;\;\textbf{1.01} & \textbf{99.46\%} & \;\;\;\textbf{0.99} & \textbf{98.69\%} &
& \textbf{96.72\%} & \;\;\;\textbf{2.29} & \textbf{95.79\%} & \;\;\;\textbf{2.06} & \textbf{58.52\%} \\

\bottomrule
\end{tabular}}
\end{table}

Next, we experiment on joint training of continuous and discrete data. Of particular interest is the ability to learn both a generative model and a classifier, referred to as a hybrid model which is useful for downstream applications such as semi-supervised learning and out-of-distribution detection~\citep{nalisnick2019hybrid}. Let $x$ be the data and $y$ be a categorical random variable. The maximum likelihood objective can be separated into $\log p(x,y) = \log p(x) + \log p(y|x)$, where $\log p(x)$ is modeled using a flow-based generative model and $\log p(y|x)$ is a classifier network that shares learned features from the generative model. However, it is often the case that accuracy is the metric of interest and log-likelihood is only used as a surrogate training objective. In this case, \citep{nalisnick2019hybrid} suggests a weighted maximum likelihood objective,
\begin{equation}
    \E_{(x,y) \sim p_{\textnormal{data}}} [\lambda \log p(x) + \log p(y|x)],
\end{equation}
where $\lambda$ is a scaling constant. As $y$ is much lower dimensional than $x$, setting $\lambda < 1$ emphasizes classification, and setting $\lambda = 0$ results in a classification-only model which can be compared against.

\begin{wraptable}[9]{r}{0.5\linewidth}
	\centering
	\ra{1.2}
	\vspace{0mm}
	\setlength{\tabcolsep}{2pt}
	\caption{Hybrid modeling results on CIFAR-10.}
	\label{tab:hybrid_cifar10}
	\resizebox{\linewidth}{!}{%
		\small
		\begin{tabular}{@{} l ccccc @{}}\toprule
			& $\lambda=0$ & \multicolumn{2}{c}{$\lambda=\nicefrac{1}{D}$} & \multicolumn{2}{c}{$\lambda=1$} \\
			\cmidrule{2-2} \cmidrule(lr){3-4} \cmidrule{5-6} 
			{\normalsize Block Type} & Acc$\uparrow$ & \;\;\;BPD$\downarrow$ & Acc$\uparrow$ & \;\;\;BPD$\downarrow$ & Acc$\uparrow$ \\
			\midrule
			{\normalsize Coupling} & 89.77\% & \;\;\;4.30 & 87.58\% & \;\;\;3.54 & 67.62\%\\
			
			{\normalsize \; + $1\times 1$ Conv} & 90.82\% & \;\;\;4.09 & 87.96\% & \;\;\;3.47 & 67.38\% \\
			
			{\normalsize Residual} & \textbf{91.78\%} & \;\;\;\textbf{3.62} & \textbf{90.47\%} & \;\;\;\textbf{3.39} & \textbf{70.32\%} \\
			\bottomrule
	\end{tabular}}
\end{wraptable}
Since \citet{nalisnick2019hybrid} performs approximate Bayesian inference and uses a different architecture than us, we perform our own ablation experiments to compare residual blocks to coupling blocks~\citep{dinh2014nice} as well as 1$\times$1 convolutions~\citep{kingma2018glow}. We use the same architecture as the density estimation experiments and append a classification branch that takes features at the final output of multiple scales (see details in Appendix \ref{app:experiment_setup}). This allows us to also use features from intermediate blocks whereas \citet{nalisnick2019hybrid} only used the final output of the entire network for classification. Our implementation of coupling blocks uses the same architecture for $g(x)$ except we use ReLU activations and no longer constrain the Lipschitz constant.

Tables~\ref{tab:hybrid} \&~\ref{tab:hybrid_cifar10} show our experiment results. Our architecture outperforms \citet{nalisnick2019hybrid} on both pure classification and hybrid modeling. Furthermore, on MNIST we are able to jointly obtain a decent classifier and a strong density model over all settings. In general, we find that residual blocks perform much better than coupling blocks at learning representations for both generative and discriminative tasks. Coupling blocks have very high bits per dimension when $\lambda = \nicefrac{1}{D}$ while performing worse at classification when $\lambda = 1$, suggesting that they have restricted flexibility and can only perform one task well at a time.

\section{Conclusion}
We have shown that invertible residual networks can be turned into powerful generative models. 
The proposed unbiased flow-based generative model, coined Residual Flow, achieves competitive or better performance compared to alternative flow-based models in density estimation, sample quality, and hybrid modeling. 
More generally, we gave a recipe for introducing stochasticity in order to construct tractable flow-based models with a different set of constraints on layer architectures than competing approaches, which rely on exact log-determinant computations.  This opens up a new design space of expressive but Lipschitz-constrained architectures that has yet to be explored.

\section*{Acknowledgments}
Jens Behrmann gratefully acknowledges the financial support from the German Science Foundation for RTG 2224 ``$\pi^3$: Parameter Identification - Analysis, Algorithms, Applications''

{
\bibliographystyle{plainnat}
\bibliography{references}
}

\onecolumn
\appendix

\section{Random Samples}\label{app:samples}

\begin{figure}[H]
    \centering
    \begin{subfigure}[b]{0.49\textwidth}
        \centering
        \includegraphics[width=\linewidth, trim= 0 0 0 35px, clip]{imgs/cifar10_real.jpg}
        \caption*{Real Data}
    \end{subfigure}%
    \hfill
    \begin{subfigure}[b]{0.49\textwidth}
        \centering
        \includegraphics[width=\linewidth, trim= 0 0 0 35px, clip]{imgs/cifar10_resflow.jpg}
        \caption*{Residual Flow}
    \end{subfigure}\\
    \vspace{1em}
    \begin{subfigure}[b]{0.49\textwidth}
        \centering
        \includegraphics[width=\linewidth]{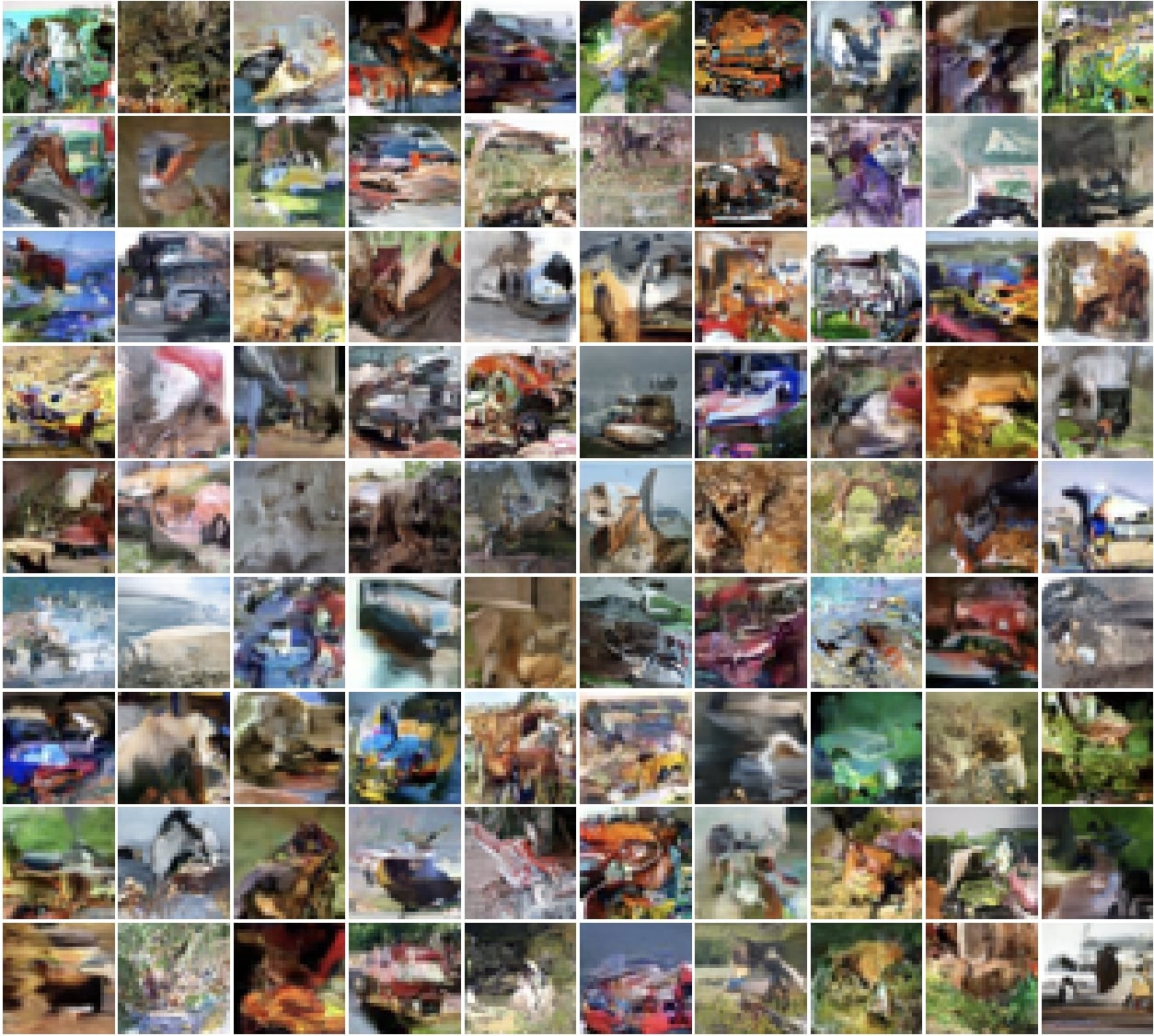}
        \caption*{PixelCNN}
    \end{subfigure}%
    \hfill
    \begin{subfigure}[b]{0.49\textwidth}
        \centering
        \includegraphics[width=\linewidth, trim= 0 0 0 34px, clip]{imgs/cifar10_flowpp.jpg}
        \caption*{Flow++}
    \end{subfigure}
    \caption{Random samples from CIFAR-10 models. PixelCNN~\citep{oord2016pixel} and Flow++ samples reprinted from \citet{ho2019flowpp}, with permission.}
\end{figure}

\pagebreak

\begin{figure}[H]
    \centering
    \begin{subfigure}[b]{0.49\textwidth}
        \centering
        \includegraphics[width=\linewidth, trim= 0 0 0 31px, clip]{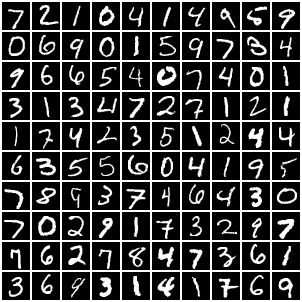}
        \caption*{Real Data}
    \end{subfigure}%
    \hfill
    \begin{subfigure}[b]{0.49\textwidth}
        \centering
        \includegraphics[width=\linewidth, trim= 0 0 0 31px, clip]{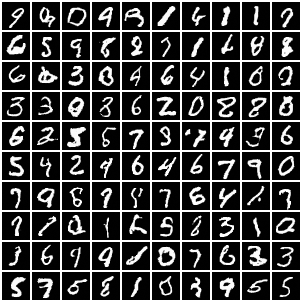}
        \caption*{Residual Flow}
    \end{subfigure}
    \caption{Random samples from MNIST.}
\end{figure}

\begin{figure}[H]
    \centering
    \begin{subfigure}[b]{0.49\textwidth}
        \centering
        \includegraphics[width=\linewidth]{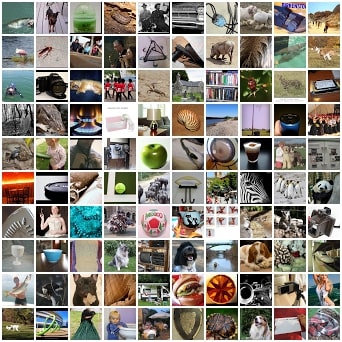}
        \caption*{Real Data}
    \end{subfigure}%
    \hfill
    \begin{subfigure}[b]{0.49\textwidth}
        \centering
        \includegraphics[width=\linewidth]{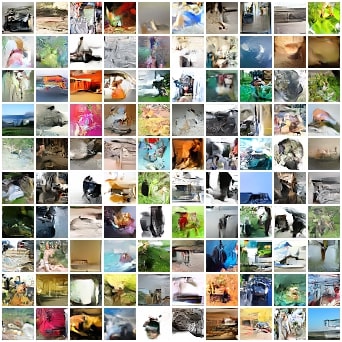}
        \caption*{Residual Flow}
    \end{subfigure}
    \caption{Random samples from ImageNet 32$\times$32.}
\end{figure}

\pagebreak

\begin{figure}[H]
    \centering
    \begin{subfigure}[b]{0.49\textwidth}
        \centering
        \includegraphics[width=\linewidth]{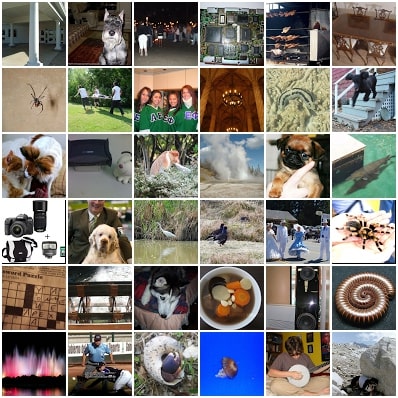}
        \caption*{Real Data}
    \end{subfigure}%
    \hfill
    \begin{subfigure}[b]{0.49\textwidth}
        \centering
        \includegraphics[width=\linewidth]{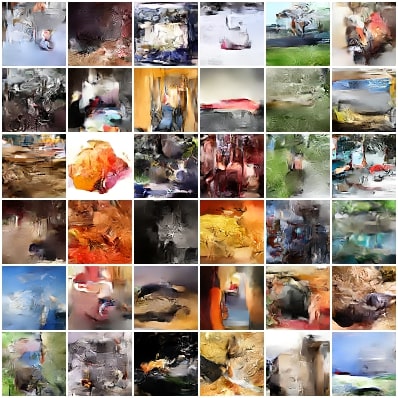}
        \caption*{Residual Flow}
    \end{subfigure}
    \caption{Random samples from ImageNet 64$\times$64.}
\end{figure}

\begin{figure}[H]
    \centering
    \begin{subfigure}[b]{0.49\textwidth}
        \centering
        \includegraphics[width=\linewidth]{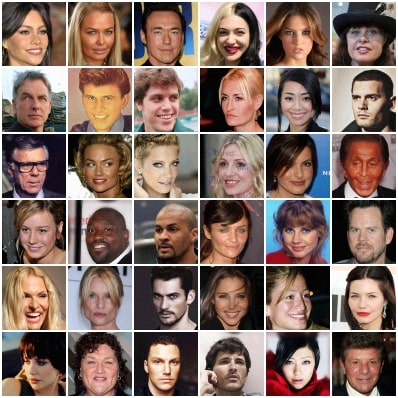}
        \caption*{Real Data}
    \end{subfigure}%
    \hfill
    \begin{subfigure}[b]{0.49\textwidth}
        \centering
        \includegraphics[width=\linewidth]{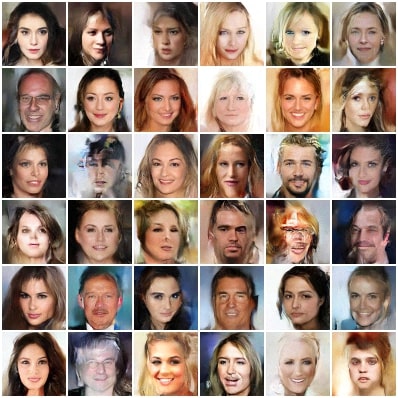}
        \caption*{Residual Flow}
    \end{subfigure}
    \caption{Random samples from 5bit CelebA-HQ 64$\times$64.}
\end{figure}

\begin{figure}[H]
	\centering
	\begin{subfigure}[b]{0.49\textwidth}
		\centering
		\includegraphics[width=\linewidth]{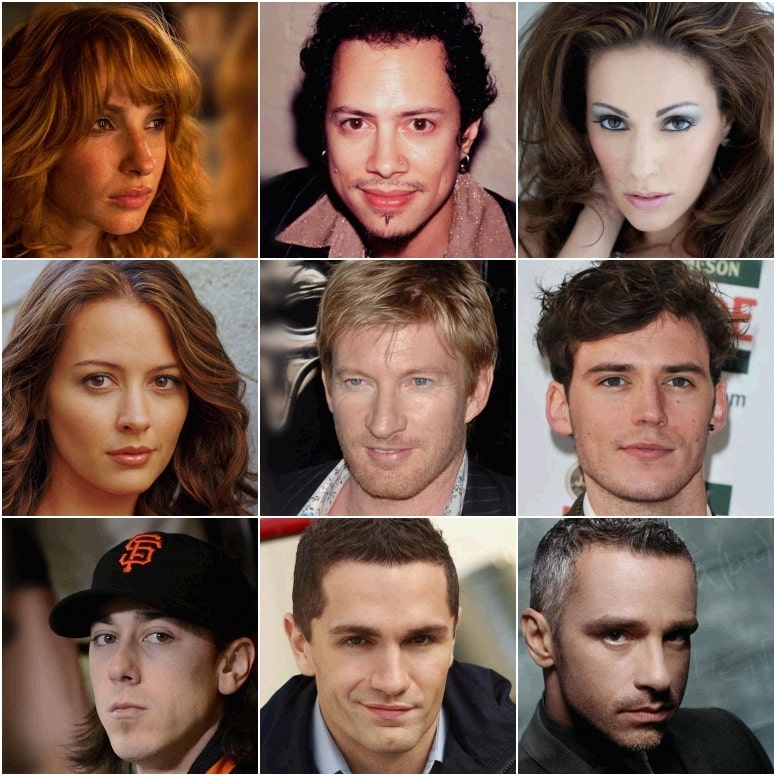}
		\caption*{Real Data}
	\end{subfigure}%
	\hfill
	\begin{subfigure}[b]{0.49\textwidth}
		\centering
		\includegraphics[width=\linewidth]{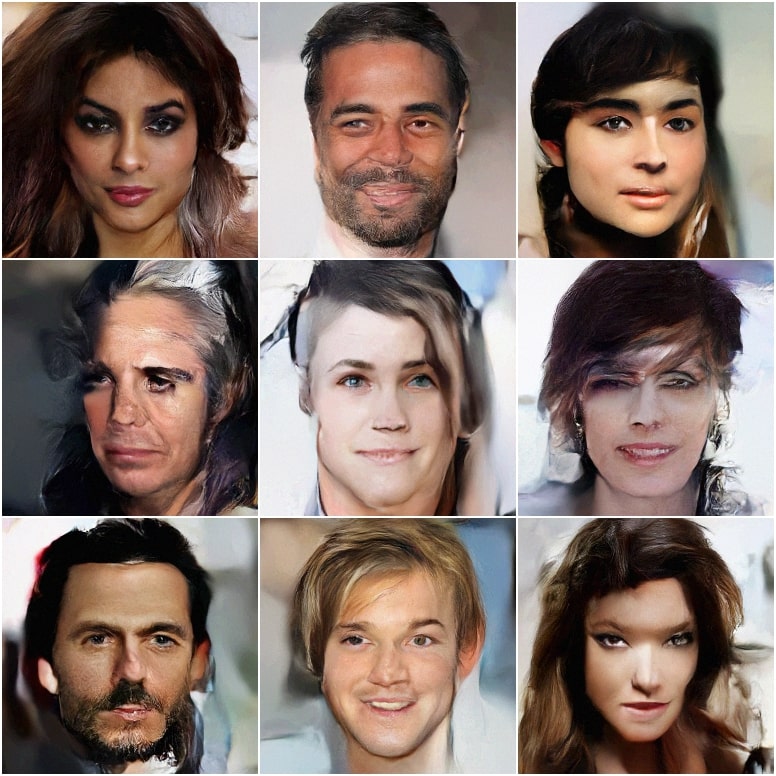}
		\caption*{Residual Flow}
	\end{subfigure}
	\caption{Random samples from 5bit CelebA-HQ 256$\times$256. Most visually appealing batch out of five was chosen.}
\end{figure}

\section{Proofs}
We start by formulating a Lemma, which gives the condition when the randomized truncated series is an unbiased estimator in a fairly general setting. Afterwards, we study our specific estimator and prove that the assumption of the Lemma is satisfied. 

Note, that similar conditions have been stated in previous works, e.g. in \citet{mcleish2010general} and \citet{rhee2012new}. However, we use the condition from \citet{bouchardcote}, which only requires $p(N)$ to have sufficient support.

To make the derivations self-contained, we reformulate the conditions from \citet{bouchardcote} in the following way:
\begin{lemma}[Unbiased randomized truncated series]
\label{lem:conditionConvergence}
Let $Y_k$ be a real random variable with $\lim_{k \rightarrow \infty} \mathbb{E}[Y_k] = a$ for some $a\in \mathbb{R}$. Further, let $\Delta_0 = Y_0$ and $\Delta_k = Y_k - Y_{k-1}$ for $k \geq 1$. Assume 
\begin{align*}
    \mathbb{E}\left[\sum_{k=0}^\infty |\Delta_k| \right] < \infty
\end{align*}
and let $N$ be a random variable with support over the positive integers and $n \sim p(N)$. Then for
\begin{align*}
    Z = \sum_{k=0}^n \frac{\Delta_k}{\mathbb{P}(N \geq k)},
\end{align*} 
it holds 
\begin{align*}
    a = \lim_{k \rightarrow \infty} \mathbb{E}[Y_k] = \mathbb{E}_{n \sim p(N)}[Z] = a.
\end{align*}
\end{lemma}
\begin{proof}
First, denote
\begin{align*}
    Z_M = \sum_{k=0}^M \frac{\mathbbm{1}[N \geq k] \Delta_k}{\mathbb{P}(N \geq k)} \quad \text{and} \quad B_M = \sum_{k=0}^M \frac{\mathbbm{1}[N \geq k] |\Delta_k|}{\mathbb{P}(N \geq k)},
\end{align*}
where $|Z_M| \leq B_M$ by the triangle inequality. Since $B_M$ is non-decreasing, the monotone convergence theorem allows swapping the expectation and limit as $\mathbb{E}[B] = \mathbb{E}[\lim_{M \rightarrow \infty}B_M] = \lim_{M \rightarrow \infty} \mathbb{E}[B_M]$. Furthermore, it is
\begin{align*}
    \mathbb{E}[B] &=\lim_{M \rightarrow \infty} \mathbb{E}[B_M] = \lim_{M \rightarrow \infty} \sum_{k=0}^M \mathbb{E}\left[ \frac{\mathbbm{1}[N \geq k] |\Delta_k|}{\mathbb{P}(N \geq k)} \right]\\ &= \lim_{M \rightarrow \infty} \sum_{k=0}^M \frac{\mathbbm{P}(N \geq k)  \mathbb{E}|\Delta_k|}{\mathbb{P}(N \geq k)}  
    = \mathbb{E}\left[ \lim_{M \rightarrow \infty} \sum_{k=0}^M |\Delta_k] \right] < \infty,
\end{align*}
where the assumption is used in the last step. Using the above, the dominated convergence theorem can be used to swap the limit and expectation for $Z_M$. Using similar derivations as above, it is
\begin{align*}
    \mathbb{E}[Z] &=\lim_{M \rightarrow \infty} \mathbb{E}[Z_M]  
    = \lim_{M \rightarrow \infty} \mathbb{E}\left[\sum_{k=0}^M \Delta_k\right] 
    = \lim_{M \rightarrow \infty} \mathbb{E}\left[Y_k\right] = a, 
\end{align*}
where we used the definition of $Y_M$ and $\lim_{k \rightarrow \infty} \mathbb{E}[Y_k] = a$.
\end{proof}

\label{sec:proofs}
\begin{proof}\textbf{(Theorem \ref{thm:unbiasedEst})}\\
To simplify notation, we denote $J := J_g (x)$. Furthermore, let
\begin{align*}
    Y_N = \E_{v} \left[ \sum_{k=1}^N \frac{(-1)^{k+1}}{k} v^T J^k v\right]
\end{align*}
denote the real random variable and let $\Delta_0 = Y_0$ and $\Delta_k = Y_k - Y_{k-1}$ for $k \geq 1$, as in Lemma \ref{lem:conditionConvergence}. To prove the claim of the theorem, we can use Lemma \ref{lem:conditionConvergence} and we only need to prove that the assumption $\mathbb{E}_v[\sum_{k=1}^\infty |\Delta_k|] < \infty$ holds for this specific case.

In order to exchange summation and expectation via Fubini's theorem, we need to prove that $\sum_{k=1}^\infty |\Delta_k| < \infty$ for all $v$. Using the assumption $\Lip(g) < 1$, it is
\begin{align*}
    \sum_{k=1}^\infty |\Delta_k| &= \sum_{k=1}^\infty \left|\frac{(-1)^{k+1}}{k} v^T J^k v\right| = \sum_{k=1}^\infty \frac{\|v^T J^k v\|_2}{k} \leq \sum_{k=1}^\infty \frac{\|v^T\|_2 \|J^k\|_2 \|v\|_2}{k} \\
    &\leq 2 \|v\|_2 \sum_{k=1}^\infty \frac{ \|J\|^k_2 }{k} \leq 2 \|v\|_2 \sum_{k=1}^\infty \frac{ \Lip(g)^k_2 }{k} = 2 \|v\|_2 \log\big(1 - \Lip(g)\big) < \infty,
\end{align*}
for an arbitrary $v$. Hence,
\begin{align}
\label{eq:deltaSeries}
    \mathbb{E}_v \left[\sum_{k=1}^\infty |\Delta_k|\right]  = \sum_{k=1}^\infty \mathbb{E}_v[ |\Delta_k|].
\end{align}

Since $\tr(A) = \mathbb{E}_v[v^T A v]$ for $v \sim \mathcal{N}(0,I)$ via the Skilling-Hutchinson trace estimator \citep{hutchinson1990stochastic, skilling1989eigenvalues}, it is
\begin{align*}
    \mathbb{E}_v[ |\Delta_k|] = \left| \frac{\tr(J^k)}{k} \right|.
\end{align*}
To show that \eqref{eq:deltaSeries} is bounded, we derive the bound

\begin{align*}
    \frac{1}{k} |\tr(J^k)| \leq \frac{1}{k} \left|\sum_{i=d}^d \lambda_i(J^k)\right| \leq \frac{1}{k} \sum_{i=d}^d |\lambda_i(J^k)| \leq \frac{d}{k}  \rho(J^k) 
     \leq \frac{d}{k} \|J^k\|_2 \leq \frac{d}{k} \Lip(g)^k,
\end{align*}
where $\lambda(J^k)$ denote the eigenvalues and $\rho(J^k)$ the spectral radius. Inserting this bound into \eqref{eq:deltaSeries} and using $\Lip(g) < 1$ yields
\begin{align*}
     \mathbb{E}_v[ |\Delta_k|]  \leq d \sum_{k=1}^\infty \frac{\Lip(g)^k}{k} = -d \, \log\big(1 - \Lip(g)\big) < \infty.
\end{align*}
Hence, the assumption of Lemma \ref{lem:conditionConvergence} is verified.
\end{proof}

\begin{proof}\textbf{(Theorem \ref{thm:unbiasedGrad})}\\
The result can be proven in an analogous fashion to the proof of Theorem \ref{thm:unbiasedEst}, which is why we only present a short version without all steps. 

By obtaining the bound
\begin{align*}
\sum_{k=0}^\infty \left| (-1)^k v^T \left( J(x, \theta)^k \frac{\partial (J_g(x,\theta))}{\partial \theta} \right) v \right| &\leq 2 \|v\|_2 \left\| \frac{\partial (J_g(x,\theta))}{\partial \theta}\right\| \sum_{k=0}^\infty \Lip(g)^k \\
&= 2 \|v\|_2 \left\| \frac{\partial (J_g(x,\theta))}{\partial \theta}\right\| \frac{1}{1-\Lip(g)} < \infty,
\end{align*}
Fubini's theorem can be applied to swap the expection and summation. Furthermore, by using the trace estimation and similar bounds as in the proof of Theorem \ref{thm:unbiasedEst}, the assumption $\mathbb{E}\left[\sum_{k=0}^\infty |\Delta_k| \right] < \infty$ from Lemma \ref{lem:conditionConvergence} can be proven. 

\end{proof}

\section{Memory-Efficient Gradient Estimation of Log-Determinant}\label{app:gradient}

Derivation of gradient estimator via differentiating power series:
\begin{align*}
\frac{\partial}{\partial \theta_i} \log\det \big(I + J_g(x,\theta)\big) &= \frac{\partial}{\partial \theta_i} \left(\sum_{k=1}^\infty  (-1)^{k+1}\frac{\tr(J_g(x,\theta)^k)}{k}\right)\\
&=  \tr \left(\sum_{k=1}^\infty   \frac{(-1)^{k+1}}{k} \frac{\partial (J_g(x,\theta)^k)}{\partial \theta_i} \right)
\end{align*}

Derivation of memory-efficient gradient estimator:
\begin{align}
& \frac{\partial}{\partial \theta_i} \log\det \big(I + J_g(x,\theta)\big) \nonumber \\ 
&= \frac{1}{\det(I+J_g(x, \theta))} \left[\frac{\partial}{\partial \theta_i} \det \big(I + J_g(x,\theta)\big)\right] \label{eq:chainRule}\\
&= \frac{1}{\det(I+J_g(x, \theta))} \left[\det(I+J_g(x, \theta)) \tr \left((I+J(x, \theta))^{-1} \frac{\partial (I + J_g(x,\theta))}{\partial \theta_i} \right)\right] \label{eq:JacobiFormula}\\ 
&= \tr \left((I+J_g(x, \theta))^{-1} \frac{\partial (I + J_g(x,\theta))}{\partial \theta_i} \right) \nonumber \\ 
&= \tr \left((I+J_g(x, \theta))^{-1} \frac{\partial (J_g(x,\theta))}{\partial \theta_i} \right) \nonumber \\
&= \tr \left( \left[\sum_{k=0}^\infty (-1)^k  J_g(x, \theta)^k \right]\frac{\partial (J_g(x,\theta))}{\partial \theta_i} \right). \label{eq:neummanSeries2}
\end{align}
Note, that \eqref{eq:chainRule} follows from the chain rule of differentiation, for the derivative of the determinant in \eqref{eq:JacobiFormula}, see \citep[eq. 46]{matrixCookbook}. Furthermore, \eqref{eq:neummanSeries2} follows from the properties of a Neumann-Series which converges due to $\|J_g(x, \theta)\| < 1$.

Hence, if we are able to compute the trace exactly, both approaches will return the same values for a given truncation $n$. However, when estimating the trace via the Hutchinson trace estimator the estimation is not equal in general:
\begin{align*}
v^T \left(\sum_{k=1}^\infty   \frac{(-1)^{k+1}}{k} \frac{\partial (J_g(x,\theta)^k)}{\partial \theta_i} \right) v
\neq v^T \left(\left[\sum_{k=0}^\infty   (-1)^{k} J_g^{k}(x,\theta)\right]\frac{\partial (J_g(x,\theta))}{\partial \theta_i} \right) v .
\end{align*}
This is due to the terms in these infinite series being different.

Another difference between both approaches is their memory consumption of the corresponding computational graph. The summation $\sum_{k=0}^\infty   (-1)^{k} J_g^{k}(x,\theta)$ is not being tracked for the gradient, which allows to compute the gradient with constant memory with respect to the truncation $n$.

\section{Generalizing Lipschitz Constraints to Induced Mixed Norms}\label{app:generalized_sn}

\citet{behrmann2019} used spectral normalization~\citep{miyato2018spectral} (which relies on power iteration to approximate the spectral norm) to enforce the Lipschitz constraint on $g$. Specifically, this bounds the spectral norm of the Jacobian $J_g$ by the sub-multiplicativity property. If $g(x)$ is a neural network with pre-activation defined recursively using $z_l = W_lh_{l-1} + b_l$ and $h_l = \phi(z_l)$, with $x=z_0, g(x)=z_L$, then the data-independent upper bound
\begin{equation}\label{eq:submult}
    ||J_g||_2 = ||W_L \phi'(z_{L-1})\cdots W_2 \phi'(z_1) W_1 \phi'(z_0)||_2 \leq ||W_1||_2 \cdots ||W_L||_2
\end{equation}
holds, where $\phi'(z)$ are diagonal matrices containing the first derivatives of the activation functions. The inequality in \eqref{eq:submult} is a result of using a sub-multiplicative norm and assuming that the activation functions have Lipschitz less than unity. However, any induced matrix norm satisfies the sub-multiplicativity property, including mixed norms $||W||_{p\rightarrow q} := \sup_{x \neq 0} \nicefrac{||Wx||_q}{||x||_p}$, 
where the input and output spaces have different vector norms.

As long as $g(x)$ maps back to the original normed (complete) vector space, the Banach fixed point theorem used in the proof of invertibility of residual blocks~\citep{behrmann2019} still holds. As such, we can choose arbitrary $p_0,\dots,p_{L-2} \in [1,\infty]$ such that
\begin{equation}\label{eq:submultMixed}
    ||J_g||_{p_0} \leq ||W_1||_{p_0\rightarrow p_1} ||W_2||_{p_1 \rightarrow p_2} \cdots ||W_L||_{p_{L-2}\rightarrow p_0}.
\end{equation}
We use a more general form of power iteration~\citep{qetlab} for estimating induced mixed norms, which becomes the standard power iteration for $p=q=2$. Furthermore, the special cases where $p_l=1$ or $p_l=\infty$ are of particular interest, as the matrix norms can be computed exactly~\citep{troppPhdThesis}. Additionally, we can also optimize the norm orders during training by backpropagating through the modified power method. Lastly, we emphasize that the convergence of the infinite series \eqref{eq:cov_res} is guaranteed for any induced matrix norm, as they still upper bound the spectral radius \citep{horn2012matrix}.



\begin{figure}[t]
    \centering
    \begin{subfigure}[b]{0.5\textwidth}
        \centering
        \includegraphics[width=\linewidth, trim=0 8px 0 10px, clip]{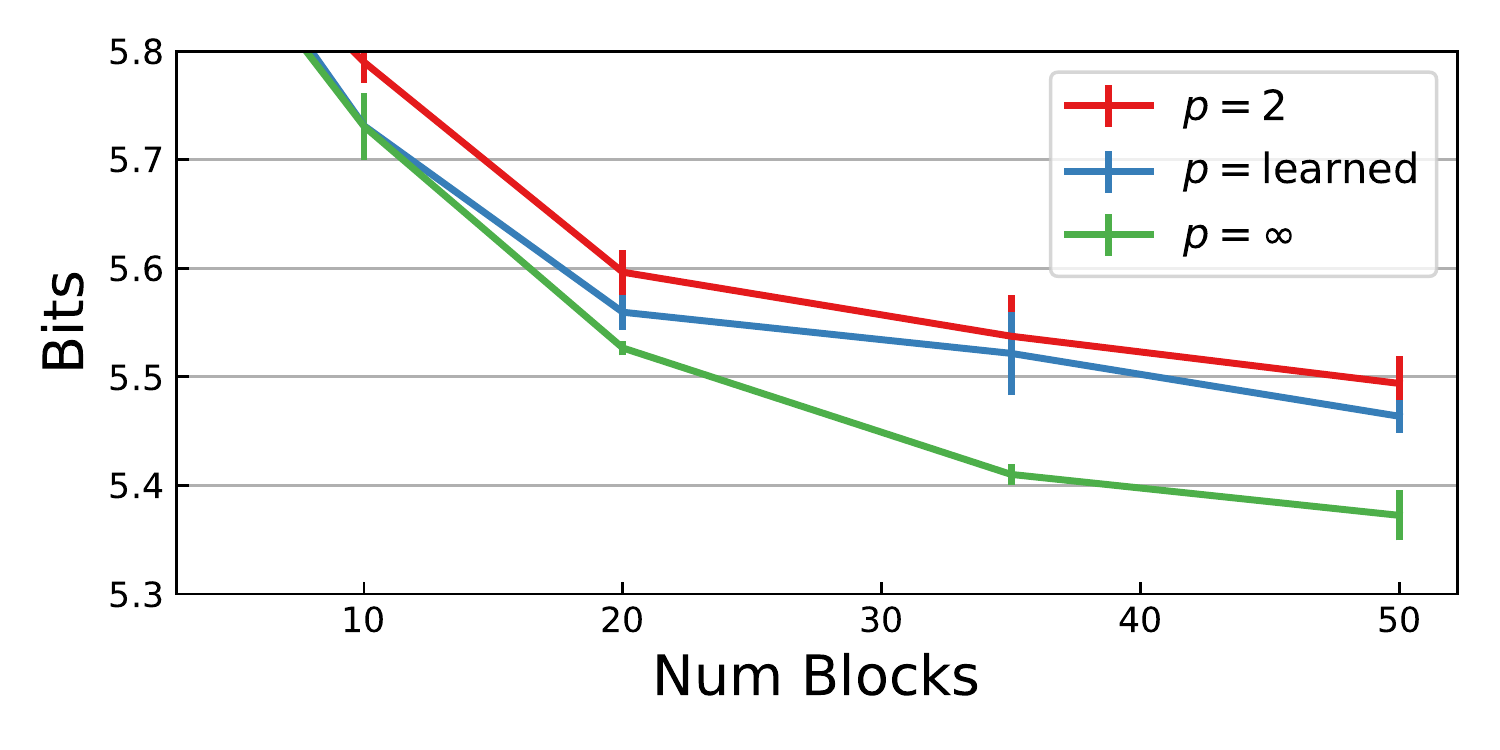}
        \caption{Checkerboard 2D}
        \label{fig:checkerboard_norms}
    \end{subfigure}%
    \begin{subfigure}[b]{0.5\textwidth}
        \centering
        \includegraphics[width=\linewidth, trim=0 8px 0 10px, clip]{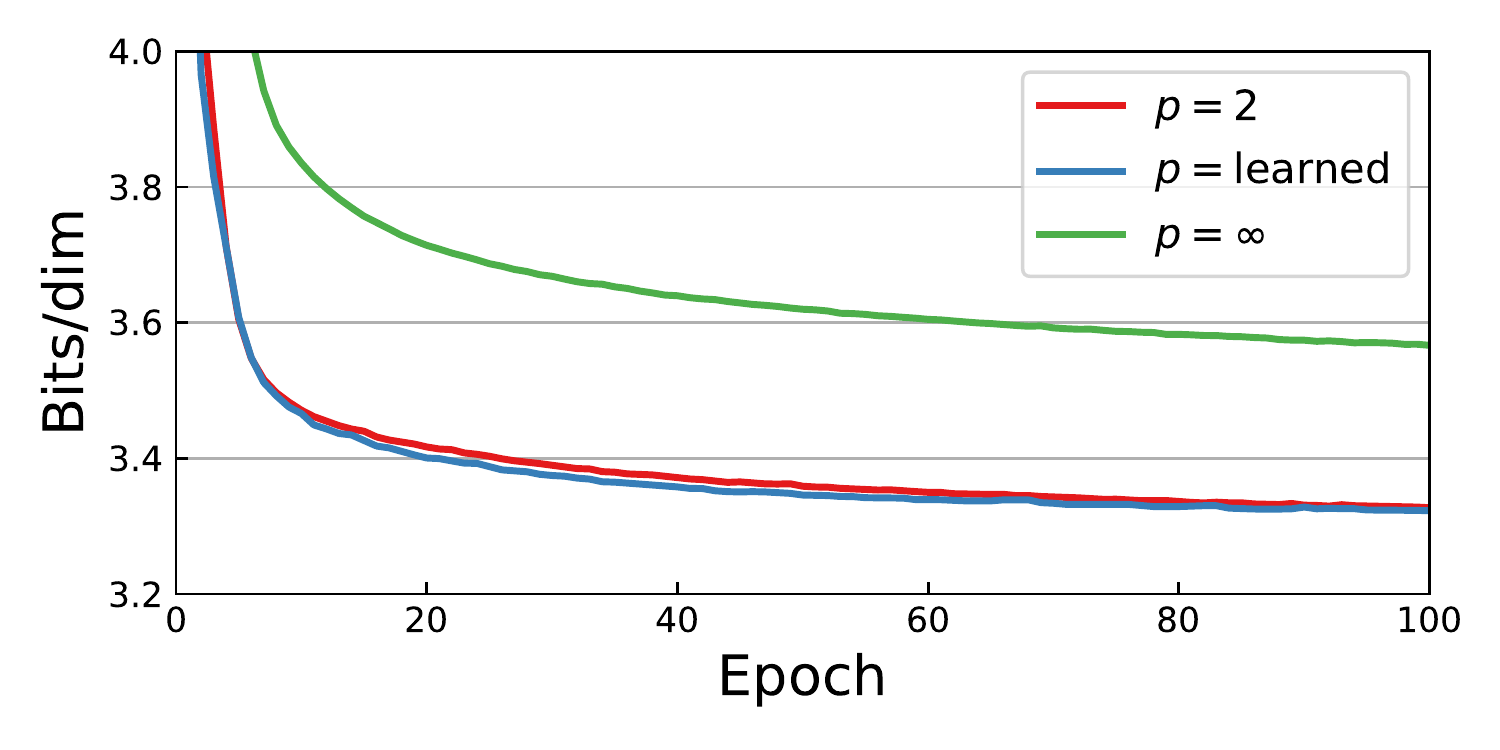}
        \caption{CIFAR-10}
        \label{fig:cifar10_norms}
    \end{subfigure}
    \caption{\textbf{Lipschitz constraints with different induced matrix norms}. (a) On 2D density, both learned and $\infty$-norm improve upon spectral norm, allowing more expressive models with fewer blocks. Standard deviation across 3 random seeds. (b) On CIFAR-10, learning the norm orders give a small performance gain and the $\infty$-norm performs much worse than spectral norm ($p=2$). Comparisons are made using identical initialization.} 
    \label{fig:norms}
\end{figure}



Figure \ref{fig:checkerboard_norms} shows that we obtain some performance gain by using either learned norms or the infinity norm on a difficult 2D dataset, where similar performance can be achieved by using fewer residual blocks. While the infinity norm works well with fully connected layers, we find that it does not perform as well as the spectral norm for convolutional networks. Instead, learned norms always perform slightly better and was able to obtain an improvement of 0.003 bits/dim on CIFAR-10. Ultimately, while the idea of generalizing spectral normalization via learnable norm orders is interesting in its own right to be communicated here, we found that the improvements are very marginal. 

\begin{figure}[H]
    \centering
    \begin{subfigure}[b]{0.2\linewidth}
    \includegraphics[width=\linewidth]{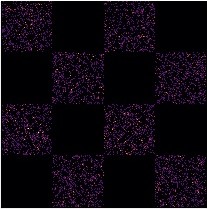}
    \caption*{Data}
    \end{subfigure}
    \begin{subfigure}[b]{0.2\linewidth}
    \includegraphics[width=\linewidth]{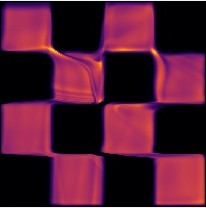}
    \caption*{$p=2$ (5.13 bits)}
    \end{subfigure}
    \begin{subfigure}[b]{0.2\linewidth}
    \includegraphics[width=\linewidth]{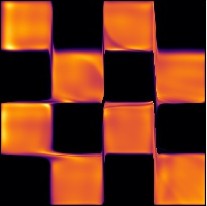}
    \caption*{$p=\infty$ (5.09 bits)}
    \end{subfigure}
    \caption{Learned densities on Checkerboard 2D.}
    \label{fig:checkerboard2d}
\end{figure}

\paragraph{Using different induced $p$-norms on Checkerboard 2D.} We experimented with the checkerboard 2D dataset, which is a rather difficult two-dimensional data to fit a flow-based model on due to the discontinuous nature of the true distribution. We used brute-force computation of the log-determinant for change of variables (which, in the 2D case, is faster than the unbiased estimator). In the 2D case, we found that $\infty$-norm always outperforms or at least matches the $p=2$ norm (ie. spectral norm). Figure~\ref{fig:checkerboard2d} shows the learned densities with 200 residual blocks. The color represents the magnitude of $p_\theta(x)$, with brighter values indicating larger values. The $\infty$-norm model produces density estimates that are more evenly spread out across the space, whereas the spectral norm model focused its density to model between-density regions.

\begin{figure}[H]
    \centering
    \includegraphics[width=0.5\linewidth, trim=0px 10px 0px 0px, clip]{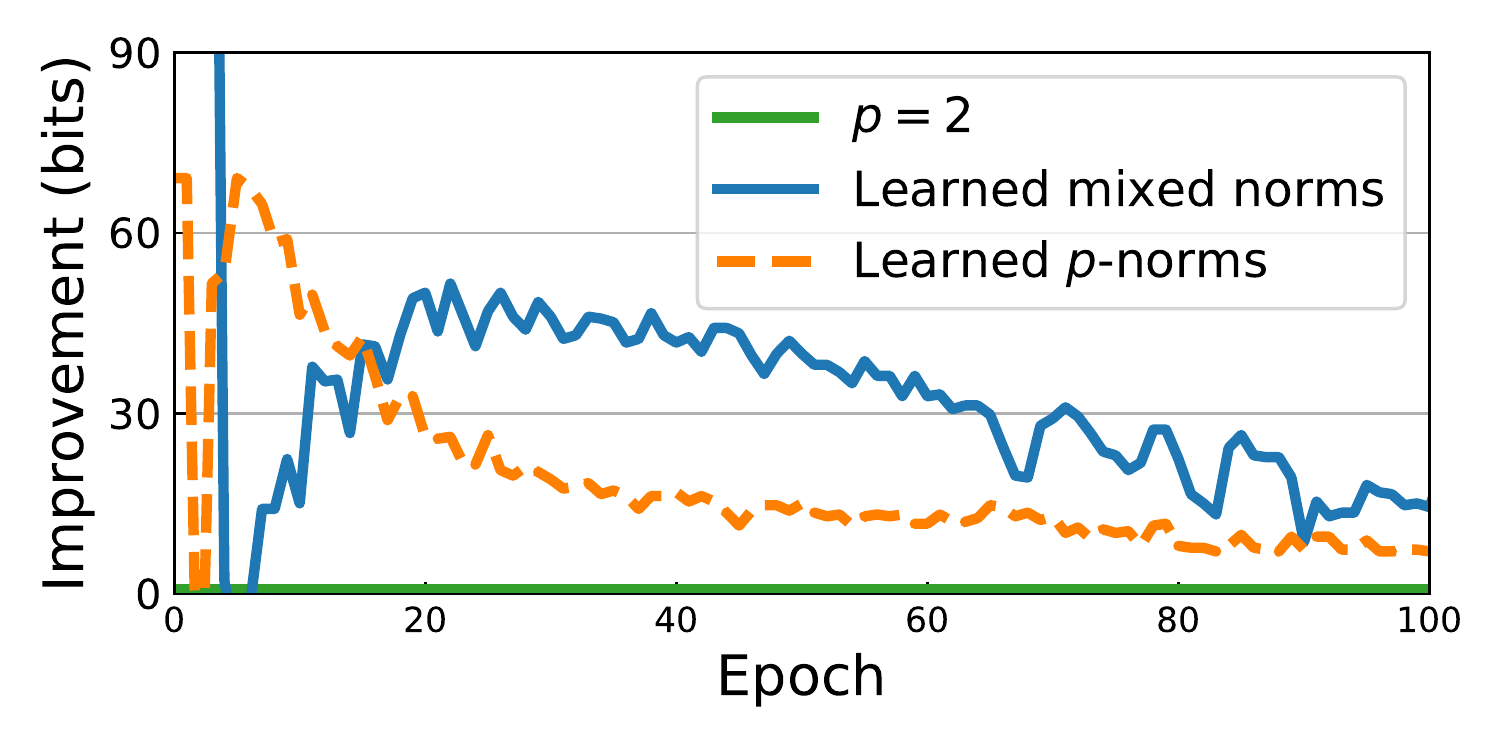}
    \caption{Improvement from using generalized spectral norm on CIFAR-10.}
    \label{fig:cifar10_learned_norms}
\end{figure}

\paragraph{Learning norm orders on CIFAR-10.}
We used $1 + \textnormal{tanh}(s) / 2$ where $s$ is a learned weight. This bounds the norm orders to $(1.5, 2.5)$. We tried two different setups. One where all norm orders are free to change (conditioned on them satisfying the constraints \eqref{eq:submultMixed}), and another setting where each states within each residual block share the same order. Figure~\ref{fig:cifar10_learned_norms} shows the improvement in bits from using learned norms. The gain in performance is marginal, and the final models only outperformed spectral norm by around $0.003$ bits/dim. Interestingly, we found that the learned norms stayed around $p=2$, shown in Figure~\ref{fig:learned_orders}, especially for the input and output spaces of $g$, ie. between blocks. This may suggest that spectral norm, or a norm with $p = 2$ is already optimal in this setting.

\begin{figure}[H]
    \centering
    \includegraphics[width=\linewidth, trim=0px 10px 0px 0px, clip]{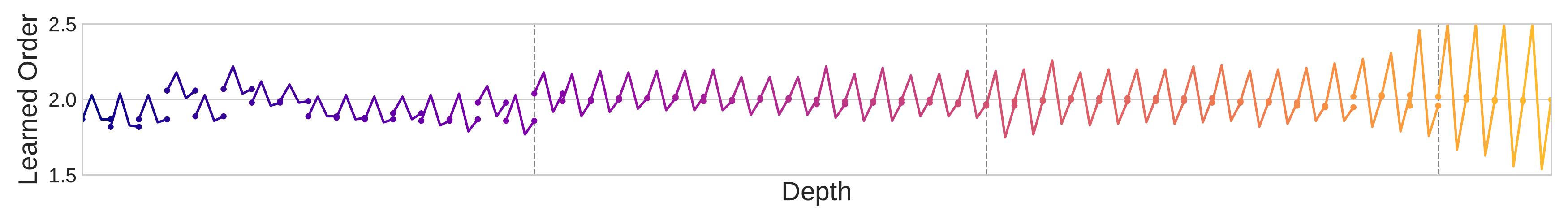}
    \caption{\textbf{Learned norm orders on CIFAR-10.} Each residual block is visualized as a single line. The input and two hidden states for each block use different normed spaces. We observe multiple trends: (i) the norms for the first hidden states are consistently higher than the input, and lower for the second. (ii) The orders for the hidden states drift farther away from 2 as depth increases. (iii) The ending order of one block and the starting order of the next are generally consistent and close to 2.}
    \label{fig:learned_orders}
\end{figure}

\section{Experiment Setup}\label{app:experiment_setup}

We use the standard setup of passing the data through a ``unsquashing'' layer (we used the logit transform~\citep{dinh2016density}), followed by alternating multiple blocks and squeeze layers~\citep{dinh2016density}. We use activation normalization~\citep{kingma2018glow} before and after every residual block. Each residual connection consists of
\begin{center}
    LipSwish $\rightarrow$ 3$\times$3 Conv $\rightarrow$ LipSwish $\rightarrow$ 1$\times$1 Conv $\rightarrow$ LipSwish $\rightarrow$ 3$\times$3 Conv
\end{center}
with hidden dimensions of $512$. Below are the architectures for each dataset.

\paragraph{MNIST.} With $\alpha=$1e-5.
\begin{center}
    Image $\rightarrow$ LogitTransform($\alpha$) $\rightarrow$ 16$\times$ResBlock $\rightarrow$  \big[ Squeeze $\rightarrow$ 16$\times$ResBlock \big]$\times$2
\end{center}

\paragraph{CIFAR-10.} With $\alpha=0.05$.
\begin{center}
    Image $\rightarrow$ LogitTransform($\alpha$) $\rightarrow$ 16$\times$ResBlock $\rightarrow$  \big[ Squeeze $\rightarrow$ 16$\times$ResBlock \big]$\times$2
\end{center}

\paragraph{SVHN.} With $\alpha=0.05$.
\begin{center}
    Image $\rightarrow$ LogitTransform($\alpha$) $\rightarrow$ 16$\times$ResBlock $\rightarrow$  \big[ Squeeze $\rightarrow$ 16$\times$ResBlock \big]$\times$2
\end{center}

\paragraph{ImageNet 32$\times$32.} With $\alpha=0.05$.
\begin{center}
    Image $\rightarrow$ LogitTransform($\alpha$) $\rightarrow$ 32$\times$ResBlock $\rightarrow$  \big[ Squeeze $\rightarrow$ 32$\times$ResBlock \big]$\times$2
\end{center}

\paragraph{ImageNet 64$\times$64.} With $\alpha=0.05$.
\begin{center}
    Image $\rightarrow$ Squeeze $\rightarrow$ LogitTransform($\alpha$) $\rightarrow$ 32$\times$ResBlock $\rightarrow$  \big[ Squeeze $\rightarrow$ 32$\times$ResBlock \big]$\times$2
\end{center}

\paragraph{CelebA 5bit 64$\times$64.} With $\alpha=0.05$.
\begin{center}
    Image $\rightarrow$ Squeeze $\rightarrow$ LogitTransform($\alpha$) $\rightarrow$ 16$\times$ResBlock $\rightarrow$  \big[ Squeeze $\rightarrow$ 16$\times$ResBlock \big]$\times$3
\end{center}

\paragraph{CelebA 5bit 256$\times$256.} With $\alpha=0.05$.
\begin{center}
	Image $\rightarrow$ Squeeze $\rightarrow$ LogitTransform($\alpha$) $\rightarrow$ 16$\times$ResBlock $\rightarrow$  \big[ Squeeze $\rightarrow$ 16$\times$ResBlock \big]$\times$5
\end{center}

For density modeling on MNIST and CIFAR-10, we added 4 fully connected residual blocks at the end of the network, with intermediate hidden dimensions of 128. These residual blocks were not used in the hybrid modeling experiments or on other datasets. For datasets with image size higher than 32$\times$32, we factored out half the variables after every squeeze operation other than the first one.

For hybrid modeling on CIFAR-10, we replaced the logit transform with normalization by the standard preprocessing of subtracting the mean and dividing by the standard deviation across the training data. The MNIST and SVHN architectures for hybrid modeling were the same as those for density modeling.

For augmenting our flow-based model with a classifier in the hybrid modeling experiments, we added an additional branch after every squeeze layer and at the end of the network. Each branch consisted of
\begin{center}
    3$\times$3 Conv $\rightarrow$ ActNorm $\rightarrow$ ReLU $\rightarrow$ AdaptiveAveragePooling(($1,1$))
\end{center}
where the adaptive average pooling averages across all spatial dimensions and resulted in a vector of dimension $256$. The outputs at every scale were concatenated together and fed into a linear softmax classifier.

\paragraph{Adaptive number of power iterations.} We used spectral normalization for convolutions~\citep{gouk2018regularisation}. To account for variable weight updates during training, we implemented an adaptive version of spectral normalization where we performed as many iterations as needed until the relative change in the estimated spectral norm was sufficiently small. As this acted as an amortization that reduces the number of iterations when weight updates are small, this did not result in higher time cost than a fixed number of power iterations, and at the same time, acts as a more reliable guarantee that the Lipschitz is bounded.

\paragraph{Optimization.} For stochastic gradient descent, we used Adam~\citep{kingma2014adam} with a learning rate of $0.001$ and weight decay of $0.0005$ applied outside the adaptive learning rate computation~\citep{loshchilov2018decoupled,zhang2018three}. We used Polyak averaging~\citep{Polyak1992AccelerationOS} for evaluation with a decay of $0.999$. 

\paragraph{Preprocessing.} For density estimation experiments, we used random horizontal flipping for CIFAR-10, CelebA-HQ 64, and CelebA-HQ 256. For CelebA-HQ 64 and 256, we preprocessed the samples to be 5bit. For hybrid modeling and classification experiments, we used random cropping after reflection padding with 4 pixels for SVHN and CIFAR-10; CIFAR-10 also included random horizontal flipping.

\end{document}